\documentclass[10pt,twocolumn,letterpaper]{article}

\usepackage[pagenumbers]{cvpr} 

\usepackage{amsthm}

\newtheorem{theorem}{Theorem}
\newtheorem{proposition}{Proposition}

\newtheorem{lemma}{Lemma}
\newtheorem{definition}{Definition}

\DeclareMathOperator*{\argmin}{argmin}

\DeclareMathOperator{\tr}{tr}
\DeclareMathOperator{\vect}{vec}

\usepackage{mathtools}
\usepackage{apptools}
\newcommand{\abs}[1]{\left\lvert#1\right\rvert}

\def\bX{{\boldsymbol{X}}}

\usepackage{multirow}

\usepackage[misc]{ifsym}

\definecolor{cvprblue}{rgb}{0.21,0.49,0.74}
\usepackage[pagebackref,breaklinks,colorlinks,allcolors=cvprblue]{hyperref}
\usepackage{bbm}
\usepackage[absolute]{textpos}
\def\paperID{13818} 
\def\confName{CVPR}
\def\confYear{2025}

\newcommand\blfootnote[1]{%
  \begingroup
  \renewcommand\thefootnote{}\footnote{#1}%
  \addtocounter{footnote}{-1}%
  \endgroup
}

\title{
Finsler Multi-Dimensional Scaling:\\ 
Manifold Learning for Asymmetric Dimensionality Reduction and Embedding
}

\author{Thomas Dag\`es\textsuperscript{1*}
\hspace{-1em} \and 
Simon Weber\textsuperscript{2,3*} 
\and Ya-Wei Eileen Lin\textsuperscript{4}
\and Ronen Talmon\textsuperscript{4}
\hspace{-1em}
\and Daniel Cremers\textsuperscript{2,3}
\hspace{-1em}
\and 
Michael Lindenbaum\textsuperscript{1}
\hspace{-1em}
\and 
Alfred M. Bruckstein\textsuperscript{1}
\hspace{-1em}
\and
Ron Kimmel\textsuperscript{1}
\hspace{-1em}
}

\begin{document}
\setlength{\abovedisplayskip}{3pt}
\setlength{\belowdisplayskip}{3pt}
{
\definecolor{somegray}{gray}{0.5}
\newcommand{\darkgrayed}[1]{\textcolor{somegray}{#1}}
\begin{textblock}{17}(-0.5, 1)  %
\centering
\darkgrayed{To appear in Proceedings of the \emph{IEEE/CVF Conference on Computer Vision and Pattern Recognition (CVPR)},\\ Nashville, TN, USA, 2025 \copyright~2025 IEEE.}
\end{textblock}
}

\maketitle

\begin{abstract}

\blfootnote{* Equal contribution. \\ 
\phantom{azer} \Letter \hspace{0.2em} {\tt thomas.dages@cs.technion.ac.il} \\
\phantom{azer} \Letter \hspace{0.2em} {\tt sim.weber@tum.de}
} 
\footnotetext[1]{Taub Faculty of Computer Science, Technion}
\footnotetext[2]{Technical University of Munich}
\footnotetext[3]{Munich Center for Machine Learning}
\footnotetext[4]{Viterbi Faculty of Electrical and Computer Engineering, Technion}

\vspace{-1em}

Dimensionality reduction is a fundamental task that aims to simplify complex data by reducing its feature dimensionality while preserving essential patterns, with core applications in data analysis and visualisation.
To preserve the underlying data structure, multi-dimensional scaling (MDS) methods focus on preserving pairwise dissimilarities, such as distances. 
They optimise the embedding to have pairwise distances as close as possible to the data dissimilarities.
However, the current standard is limited to embedding data in Riemannian manifolds.
Motivated by the lack of asymmetry in the Riemannian metric of the embedding space, this paper extends the MDS problem to a natural asymmetric generalisation of Riemannian manifolds called Finsler manifolds. 
Inspired by Euclidean space, we define a canonical Finsler space for embedding asymmetric data.
Due to its simplicity with respect to geodesics, data representation in this space is both intuitive and simple to analyse. 
We demonstrate that our generalisation benefits from the same theoretical convergence guarantees.
We reveal the effectiveness of our Finsler embedding across various types of non-symmetric data, highlighting its value in applications such as data visualisation, dimensionality reduction, directed graph embedding, and link prediction.

\vspace{-1em}

\end{abstract}

\section{Introduction}

Dimensionality reduction is a task at the core of data analysis, allowing complex data to be represented in simple low-dimensional forms.
In Multi-Dimensional Scaling (MDS), the goal is to preserve pairwise dissimilarities between datapoints, thereby retaining the essential structure of the data.
\begin{figure}[t]
  \centering
   \includegraphics[width=\columnwidth]{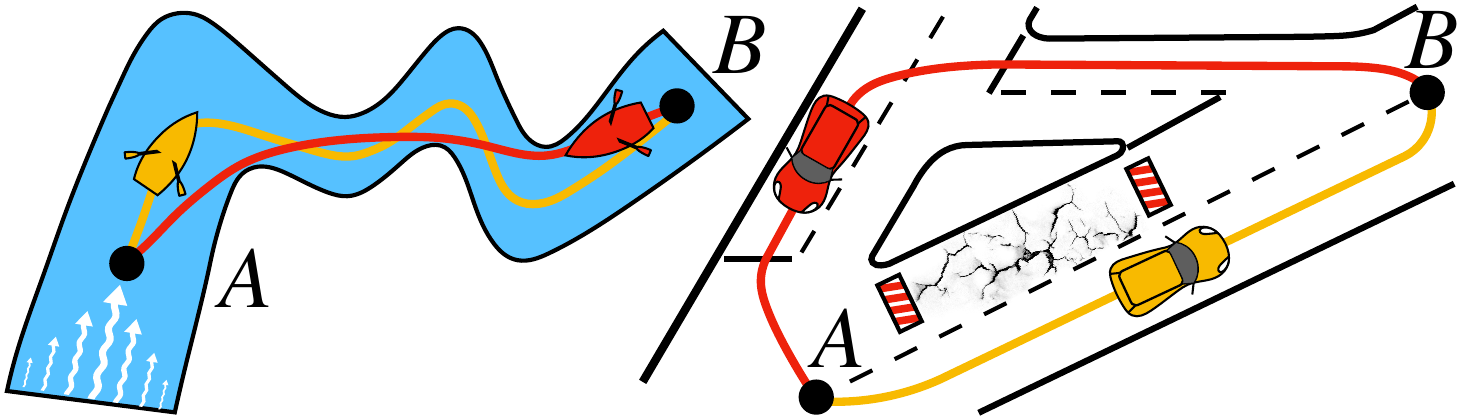}
   \caption{
   External fields, such as non-uniform currents, lead to time-wise geodesic curves from $A$ to $B$ (in yellow) that differ from those from $B$ to $A$ (in red). 
   Although asymmetry is impossible in Riemannian manifolds, it is at the core of Finsler geometry.
   }
   \label{fig: Finsler river road}
   \vspace{-1em}
\end{figure}
However, the original MDS approach \cite{shepard1962analysis,kruskal1964multidimensional,cox2000multidimensional} assumes that data relationships are primarily linear, limiting its ability to capture the non-linear structures often present in high dimensions.
Manifold learning extends this approach by assuming that data lies on a low-dimensional manifold, where meaningful relationships are reflected in more complex pairwise dissimilarities.
Methods that integrate manifold learning into MDS can capture these structures by focusing on pairwise dissimilarities, such as geodesic distances, which reflect the intrinsic geometry of the manifold and better preserve its continuity in the embedding space.

The MDS community has universally decided to embed data into spaces implicitly equipped with a Riemannian geometry \cite{bronstein2008numerical,do2016differential}, choosing Euclidean \cite{schwartz1989numerical,tenenbaum2000global} or more generally geodesic \cite{bronstein2006generalised} distances in embedding space.
However, such embeddings inherit from the limitations of Riemannian metrics. In particular, Riemannian distances are bound to be symmetric, meaning that they cannot accurately capture the possible asymmetric nature of the original data dissimilarities.
Yet, such properties are common in real-world applications across a wide range of fields. 
In physical systems, external fields such as currents lead to different shortest-time paths whether going upstream or downstream (see \cref{fig: Finsler river road}). 
In directed graphs, from social media to road networks, non-symmetric edges lead to non-symmetric pairwise dissimilarities.
This inability of Riemannian spaces to account for asymmetry suggests that asymmetric data should be embedded in spaces equipped with other types of metrics than the traditional Riemannian ones.
Yet, challenging the Riemannian metric of the embedding space remains largely uncharted.

This strong limitation calls for defining and studying the MDS problem where the embedding space is equipped with a more general metric. 
Bernard Riemann himself first mentioned generalisations to his own metrics \cite{riemann1854hypotheses}, but they were only investigated later by Paul Finsler. In 1918, he proposed a generalisation of Riemannian manifolds \cite{finsler1918ueber}, which Elie Cartan later named \emph{Finsler manifolds} \cite{cartan1933espaces}.
The main assumption in Finsler manifolds is that the metric is not symmetric, unlike in Riemannian ones, meaning that distances are sensitive to the direction of motion. 
Consequently, both local and global distances between two points are asymmetric\footnote{Asymmetric metrics and distances are sometimes called quasi metrics and quasi distances in the literature.}; the distance from $A$ to $B$ is generally different from the distance from $B$ to $A$. 
Finsler manifolds have been mostly investigated in theoretical mathematics \cite{shen2001lectures,bao2012introduction,shen2016introduction}, gravitational physics \cite{randers1941asymmetrical,vacaru2005clifford,vacaru2011principles,pfeifer2012finsler,asanov2012finsler,caponio2014wind,hohmann2019finsler}, mechanics \cite{bidabadi2010geometric,clayton2015finsler}, current navigation problems \cite{zermelo1931navigationsproblem}, and seismic sciences \cite{yajima2009finsler}.
However, they are largely unexplored in the computer vision and machine learning community.

In this paper, we generalise the MDS problem to account for asymmetric dissimilarities by embedding the data into a Finsler space rather than a symmetric Riemannian one.
In particular, our contributions are as follows:

\begin{enumerate}[label=(\alph*)]
    \setlength\itemsep{0.3em}
    \item We revisit multi-dimensional scaling by formulating a new problem, called \emph{Finsler MDS}. 
    The goal is to find an embedding to a Finsler space, allowing asymmetric data dissimilarities. 
    To our knowledge, our framework is the first that goes beyond Riemannian geometry in the field of manifold learning.

    \item We propose a new canonical Finsler space, that generalises the Euclidean space to incorporate asymmetric distances. 
    Its geodesic simplicity provides embeddings that are intuitive and simple to analyse.

    \item We extend the popular SMACOF algorithm from traditional MDS to solve the Finsler MDS problem, and prove its theoretical convergence.
    We also show how to solve Finsler MDS with modern deep learning techniques and neural networks.

    \item We experimentally highlight the benefits of Finsler MDS for a variety of applications.
    We perform qualitative data visualisation experiments on various types of data -- curved manifolds, flat current maps, and directed graphs. 
    Such asymmetric visualisations are impossible in traditional MDS. 
    We also demonstrate the quantitative superiority of our method for node embedding and link prediction in directed graphs.

    \item We release our code to facilitate further research: 
    \url{https://github.com/Tommoo/FinslerMDS}.
\end{enumerate}

\section{Related Work}
\label{sec:related_work}

Multi-dimensional scaling and manifold learning have a rich history, with many different approaches to find embeddings satisfying different properties.
This contrasts with Finsler geometry which has seldom been applied in the computer vision and pattern recognition community.

\hfill \break
\noindent \textbf{Multi-Dimensional Scaling.}
MDS methods aim to preserve data dissimilarities of 
a low-dimensional embedding by minimising the stress function, which is a weighted sum of the squared difference between the embedded distances and the data dissimilarities \cite{leeuw1977application, borg2007modern,saeed2018survey}.
Dissimilarities can be computed using either local \cite{roweis2000nonlinear,donoho2003hessian,zhang2006mlle,belkin2001laplacian,belkin2003laplacian,zhang2004principal,coifman2005geometric,coifman2006diffusion,lin2023hyperbolic,lin2024tree,suzuki2019hyperbolic, gou2023discriminative,martin2005visualizing,mcinnes2018umap,venna2010information} or global \cite{weinberger2005nonlinear,weinberger2006unsupervised,brand2005nonrigid,funke2020low} criteria or a combination of the two \cite{pearson1901pca,hotelling1933analysis,scholkopf1998nonlinear,schwartz1989numerical,tenenbaum2000global,rosman2010nonlinear,aflalo2013spectral,schwartz2019intrinsic,bracha2024wormhole}.
Nevertheless, traditional approaches cannot handle asymmetric dissimilarities, which naturally occurs for many features \cite{tversky1977features} in physical systems and directed graphs.
Remarkably, few efforts have tackled asymmetry for MDS, such as
ASYMSCAL-like methods
\cite{young1975asymmetric,chino1978graphical,chino2011asymmetric,martin2005visualizing,olszewski2024asymmetric}, slide vector models \cite{leeuw1982theory,gower1977analysis}, Gower models \cite{gower1977analysis,constantine1978graphical}, the radius-distance model \cite{okada1987nonmetric,tanioka2018asymmetric}, and the hill-climbing model \cite{borg2007modern,okada2012bayesian}.
In brief, these methods use one of the following strategies. The first is to design asymmetric weights, yielding asymmetric importance for fitting to (symmetrised) dissimilarities in the stress function. The second is to design asymmetric embeddings directly by departing from a metric perspective, either by using non metric formula, as in hill-climbing, or by using additional stratagems to display asymmetric visualisations, as in the radius-distance model.
To the best of our knowledge, the embedding space has never been equipped with a metric structure that naturally reflects asymmetry, such as Finsler geometry.

\hfill \break
\noindent \textbf{Finsler Geometry and Computer Vision.}
Except for a few notable works in robotics \cite{ratliff2021}, image processing \cite{chen2016geoevo,chen2015global,chen2017elastica,chen2016minpath,yang2019geodesic,dages2024metric}, and shape analysis \cite{weber2024finsler},
Finsler geometry is a largely uncharted topic within the computer vision and machine learning community. For a mathematical introduction to Finsler geometry, see \cite{ohta2009heat,bao2012introduction,mirebeau2014efficient,bonnans2022linear}.

\section{Traditional Multi-Dimensional Scaling}

MDS aims to find an embedding $x_1,\ldots, x_N\in\mathbb{R}^m$, or in row-stacked format $\boldsymbol{X}\in\mathbb{R}^{N\times m}$, of $N$ data points $\tilde{\boldsymbol{X}}\in\mathbb{R}^{N\times n}$ into a canonical space, e.g.\ the Euclidean space $\mathbb{R}^m$ with $n>m$, while preserving data dissimilarities $\boldsymbol{D}\in\mathbb{R}_+^{N\times N}$. 
The points $\tilde{\boldsymbol{X}}$ often lie on a manifold $\tilde{\mathcal{X}}\subset \mathbb{R}^n$ to be preserved when embedding, which is the goal of \emph{manifold learning}.
If $\tilde{\mathcal{X}}$ is equipped with a metric $\tilde{L}$, given at points $\tilde{x}\in\tilde{\mathcal{X}}$ by $\tilde{L}_{\tilde{x}}:\mathcal{T}_{\tilde{x}}\tilde{\mathcal{X}}\to\mathbb{R}_+$ where $\mathcal{T}_{\tilde{x}}\tilde{\mathcal{X}}$ is the tangent plane, a manifold preserving strategy is to take geodesic distances as dissimilarities. Let us properly define this notion.

\hfill \break
\noindent \textbf{Geodesic Distances.}
Let $\mathcal{X}$ be a manifold equipped with a metric $L$ given by $L_x:\mathcal{T}_x\mathcal{X} \to \mathbb{R}_+$.
By integrating the length of tangent vectors, we can define the length of a smooth curve $\gamma:[0,1]\to \mathcal{X}$ as
\begin{equation}
    \label{eq:length}
    \mathcal{L}_L(\gamma) = \int_0^1 L_{\gamma(t)}(\gamma'(t))dt.
\end{equation}
A shortest path, also known as a shortest geodesic, $\gamma_{x \to y}^L$ between two points $x,y \in \mathcal{X}$ minimises \cref{eq:length}.
The geodesic distance from $x$ to $y$ is the length of such shortest paths
\begin{equation}
    \label{eq:distance_geodesic}
    d_L (x,y) = \mathcal{L}_L (\gamma_{x \to y}^{L}).
\end{equation}

\hfill \break
\noindent \textbf{MDS and Riemannian Geometry.}
In MDS, only the knowledge of $\boldsymbol{D}$ is needed: 
the original data $\tilde{\boldsymbol{X}}$, its dimensionality $n$, or its manifold $\tilde{\mathcal{X}}$ can be unknown.
A key assumption in MDS is that dissimilarities are symmetric, i.e.\ $\boldsymbol{D}^\top = \boldsymbol{D}$,  often implying that $\boldsymbol{D}$ represents Riemannian pairwise distances, e.g.\ Riemannian geodesic distances in $\tilde{\mathcal{X}}$. The MDS task can then be termed \emph{metric MDS}.
In consequence, the embedding space $\mathcal{X} = \mathbb{R}^m$ is viewed as a Riemannian manifold, with metric $R$, that is the embedding of the Riemannian manifold $\tilde{\mathcal{X}}$, with metric $\tilde{R}$.

\hfill \break
\noindent \textbf{Riemannian Manifolds.}
A Riemannian manifold $\mathcal{X}$ is equipped with a quadratic metric $R$ given by $R_x(u)= \sqrt{u^{\top} M(x)u}$. 
The metric tensor $M(x)$ is a symmetric positive definite (SPD) matrix that fully describes $R$ and is abusively called the Riemannian metric.
Fundamentally, Riemannian metrics are symmetric as $R_{x}(-u) = R_{x}(u)$ for all $u \in \mathcal{T}_{x}\mathcal{X}$. 
Thus, traversing a curve in one direction or the other leads to the same distance, implying symmetric geodesic distances, i.e.\ $d_R(x,y) = d_R(y,x)$.

To reveal the simple intrinsic nature of the Riemannian manifold $\tilde{\mathcal{X}}$, the MDS embedding space $\mathcal{X}$ should be a simple canonical Riemannian space. In particular for the task of \emph{manifold flattening}, geodesics in $\mathcal{X}$ should be the usual Euclidean segment\footnote{This segment from $x_i$ to $x_j$ is given by $(1-t)x_i + tx_j$ for $t\in[0,1]$.}, making the space $\mathcal{X}$ flat. 
The appropriate canonical Riemannian space is thus the traditional Euclidean space $\mathbb{R}^m$, corresponding to taking $R_x(u) = \lVert u\rVert_2$.
We can now formally present the MDS task.

\hfill \break
\noindent \textbf{MDS Formulation.}
The MDS task consists in finding the Euclidean embedding $\boldsymbol{X}$ that minimises the stress function,
\begin{equation}
    \label{eq:MDS}
    \sigma_E^{2}(\boldsymbol{X}) = \sum\limits_{i,j} w_{i,j} \big(d_E(x_i,x_j) - \boldsymbol{D}_{i,j}\big)^2,
\end{equation}
where
$d_E$ is the Euclidean distance
in $\mathbb{R}^m$ of the embedded data $\boldsymbol{X}$
and $\boldsymbol{D}$ is a symmetric dissimilarity matrix.
The weights $w_{i,j}\ge 0$ are usually taken to be uniform, however other configurations can be useful to better preserve the topology of the manifold, such as when holes are present due to undersampling or boundaries \cite{bronstein2006generalised,bronstein2006efficient,rosman2010nonlinear,schwartz2019intrinsic,bracha2024wormhole}.

Kernel-based \cite{scholkopf1997kernel} closed form solutions exist for uniform weights $w_{i,j}$, and it is called Isomap when Dijkstra's algorithm \cite{dijkstra1959note} is used to compute geodesic distances as dissimilarities.
This property breaks down for non-uniform weights in which case the stress is minimised using an iterative descent strategy, such as the SMACOF algorithm \cite{leeuw1977application}.

\hfill \break
\noindent \textbf{MDS Limitations.}
Traditional MDS requires symmetric dissimilarities, i.e.\ $\boldsymbol{D}^\top = \boldsymbol{D}$. 
However, in many cases, dissimilarities are naturally non-symmetric, i.e.\ $\boldsymbol{D}^\top \neq \boldsymbol{D}$.
For example, this arises with the non-symmetric weights of directed graphs, but it can also occur in real-world systems (see \cref{fig: Finsler river road}).
Yet, to the best of our knowledge, none of these methods fully address a core limitation of traditional MDS: its reliance on a Riemannian embedding space, which imposes symmetric distances between points. Thus, even with various adjustments, embedding into a Riemannian space $\mathcal{X}$ cannot capture the non-symmetric structure of the data.
We address this by introducing metric tools that move beyond Riemannian geometry and support \emph{asymmetric} distances. Specifically, we consider Finsler geometry as it offers a refined, asymmetric generalisation of Riemannian geometry. 
Our goal is to embed data with non-symmetric dissimilarities $\boldsymbol{D}$ into a Finsler space.

\section{Finsler Geometry}
\label{sec: finsler geometry}

Finsler manifolds \cite{bao2012introduction} are a generalisation of Riemannian manifolds that incorporates distance asymmetry \cite{chern1996finsler}. 
A Finsler manifold $\mathcal{X}$ is equipped with a Finsler metric $F$ satisfying at all points $x\in \mathcal{X}$ the following properties.
\begin{definition}[Finsler metric]
    A Finsler metric $F$, given by $F_x:\mathcal{T}_x \mathcal{X}\to \mathbb{R}_+$,
    is smooth,  
    positive definite \ref{it: finsler metric positive definite}, positive homogeneous \ref{it: finsler metric positive homogeneous}, 
    and satisfies the triangle inequality \ref{it: finsler metric triangular inequality}:
    \begin{enumerate}[label=(\roman*)]
        \item\label{it: finsler metric positive definite} $F_x(u)\!\ge\! 0$ and $F_x(u)=0\!\!\iff\!\! u=0$, \quad $\forall u\in \mathcal{T}_x\mathcal{X}$,
        \item\label{it: finsler metric positive homogeneous} $F_x(\lambda u) = \lambda F_x(u)$, \quad $\forall u\in \mathcal{T}_x\mathcal{X}$, $\lambda>0$,
        \item\label{it: finsler metric triangular inequality} $F_x(u+v)\le F_x(u) + F_x(v)$, \quad $\forall u,v\in \mathcal{T}_x\mathcal{X}$.
    \end{enumerate}
\end{definition}

In contrast to Riemannian metrics, Finsler metrics are no longer homogeneous, leading to metric asymmetry as $F_x(-u)\neq F_x(u)$ in general. This asymmetry is impossible for Riemannian metrics.
Additionally, Finsler metrics are no longer given by a quadratic form in general \cite{chern1996finsler} and do not share a universal parametric description.
However, several parametric families of Finsler metrics exist \cite{javaloyes2011definition}, such as the $(\alpha,\beta)$ metrics \cite{matsumoto1972c}, which induce Randers \cite{randers1941asymmetrical}, Matsumoto \cite{matsumoto1989slope} or Kropina metrics \cite{kropina1959projective}. Due to the simplicity of its explicit formulation, we consider in the following the Randers metric which generalises the Riemannian metric.

\begin{definition}[Randers metric]
       A Randers metric $F$ is parametrised by a tensor field $M$ of symmetric positive definite matrices and a drift vector field $\omega$ with  $\lVert\omega(x) \rVert_{M^{-1}(x)}\!\!<\!\!1$ such that
       $F_x(u) = \sqrt{u^{\top} M(x) u} + \omega(x)^{\top} u$.
\end{definition}

In particular, the Randers metric is the sum of a Riemannian metric associated to the tensor field $M$ and an asymmetric linear part due to the vector field $\omega$. Note that the positivity of the metric stems from $\lVert \omega(x) \rVert_{M(x)^{-1}} < 1$.
Riemannian metrics are Randers metrics, and thus also Finsler, and correspond to $\omega\equiv 0$.
We can now present the novel Finsler multi-dimensional scaling problem.

\section{Finsler Multi-Dimensional Scaling}\label{sec:Finsler Multi-Dimensional Scaling}

We aim to embed non-symmetric data into $\mathbb{R}^m$ viewed as a Finsler space. This task requires the formulation of the general Finsler MDS problem (\cref{subsec:problem_statement}) and the establishment of a canonical Finsler space (\cref{subsec:canonical_randers_space}) that generalises the canonical Riemannian space of traditional MDS.

\subsection{Problem Statement}
\label{subsec:problem_statement}

We formalise Finsler manifold learning, or \emph{Finsler MDS}, as the minimisation of the Finsler stress function
\begin{equation}
    \label{eq:fmds}
    \sigma_F^{2}(\boldsymbol{X}) = \sum\limits_{i,j} w_{i,j} \big(d_F(x_i,x_j) - \boldsymbol{D}_{i,j}\big)^2,
\end{equation}
where 
$d_F$ is the Finsler distance
in $\mathcal{X} = \mathbb{R}^m$
with chosen Finsler metric $F$. In general, unlike in classical scaling, data distances need not be symmetric, $\boldsymbol{D}_{j,i} \neq \boldsymbol{D}_{i,j}$, nor the embedding distance, $d_{F}(x_{i},x_{j}) \neq d_{F}(x_{j}, x_{i})$.

As in traditional MDS, it is often desirable for the embedding space to have a straightforward, canonical metric, ensuring that the embedding space $\mathcal{X} = \mathbb{R}^m$ remains flat.
In this context, the embedding space  $\mathbb{R}^m$ is said to be flat when its geodesic curves are simply straight Euclidean segments.
For the Riemannian case, this is achieved with the Euclidean metric. However, in Finsler spaces, we need to establish an appropriate canonical metric.
In this paper, we develop a canonical Finsler space that extends the Euclidean space, allowing for asymmetric distances. This approach enables us to maintain the simplicity and flatness while accommodating the asymmetric structure of Finsler geometry.

\subsection{Canonical Randers Space}\label{subsec:canonical_randers_space}

We propose using $\mathbb{R}^m$ as the canonical Finsler space equipped with the following Randers metric.

\begin{definition}[Canonical Randers space]
    The space $\mathbb{R}^{m}$ is a canonical Randers space if it is equipped with the canonical Randers metric $F^{C}$, with Euclidean Riemannian component and a constant drift vector $\omega$
    with $\alpha \equiv \lVert \omega\rVert_2 < 1$.
\end{definition}

The canonical Randers metric $F^C$ is uniform as
\begin{equation}\label{eq:canonical randers}
    F_x^C(u) = \lVert u \rVert_{2} + \omega^\top u, \quad \forall x\in \mathbb{R}^m.
\end{equation}
The canonical Randers space is Euclidean along hyperplanes orthogonal to $\omega$, with an additional dimension along $\omega$ to capture asymmetry. 
The parameter $\alpha$ controls the influence of asymmetry: when $\omega \equiv 0$, i.e.\ $\alpha=0$, asymmetry is discarded, and the metric reduces to the Euclidean one. 
As $\alpha$ increases, asymmetry takes on a greater role.

In Euclidean space, geodesics are given by usual straight segments. The following theorem shows that this property also holds in the canonical Randers space. 

\begin{theorem}[Flatness of the canonical Randers space]
    \label{th: canonical Randers space is flat}
    The shortest path between any points $x, y\in \mathbb{R}^m$ for the canonical Randers metric $F^C$ is the Euclidean straight segment between $x, y$, i.e.\ $\gamma_{x\to y}^{F^C}(t) = (1-t)x + ty$ for $t\in[0,1]$. 
\end{theorem}

We provide a proof in \cref{sec: proof of canonical Randers space is flat}.
\Cref{th: canonical Randers space is flat} means that the canonical Randers space is \emph{flat}, as geodesic paths are not curved.
This peculiarity results from the uniformity of the drift component.
While geodesic paths in the canonical Randers space are identical to those in Euclidean space, their lengths vary depending on the direction of traversal.


\begin{proposition}[Canonical Randers distance]
    \label{prop:randers_distance}
    The canonical Randers distance
    $d_{F^C}$
    between points $x,y\in \mathbb{R}^m$ is
    \begin{equation}\label{eq:randers distance}
        d_{F^C}(x,y) = \lVert y-x\rVert_2 + \omega^{\top} (y-x) .
    \end{equation}
\end{proposition}

See \cref{sec: proof of randers canonical distance} for a proof.
When $\omega \equiv 0$, we recover the Euclidean distance. 
However, for $\omega \neq 0$, this distance becomes \emph{asymmetric}: $d_{F^C}(y, x) \neq d_{F^C}(x, y)$.
Remarkably, geodesic distances in the canonical Randers metric have a straightforward closed-form solution.
It is composed of the Euclidean distance $\lVert y-x\rVert_2$ and an additional asymmetric term $\omega^\top (y-x)$, that is the projection of $(y-x)$ along $\omega$.
Thus, in the canonical Randers space, both shortest paths and geodesic distances are known and easy to compute, providing an intuitive visual understanding of the asymmetric structure of $\mathbb{R}^m$ viewed as a canonical Randers space.

\hfill \break
\noindent \textbf{Riemannian Generalisations.}
As noted above, the canonical Randers space generalises the Euclidean space in many aspects. When $\omega\equiv 0$, it becomes Euclidean. When $\omega\neq 0$, it still shares the same geodesic paths as the Euclidean space, and a similar simplicity for computing their length. 
Moreover, for symmetric data and non-zero  $\omega$, the canonical Randers space also extends the traditional MDS embedding, as formalised in the following theorem.

\begin{theorem}
    \label{th: canonical randers generalises riemann with extra dim}
    Assume that the data is symmetric, $\boldsymbol{D} = \boldsymbol{D}^\top$, and that it can be accurately embedded into a Euclidean space $\mathbb{R}^m$, i.e.\ we can find $\boldsymbol{X}\in\mathbb{R}^m$ such that 
    $d_E(x_i, x_j) = \boldsymbol{D}_{i,j}$ for all $(i,j)$.
    The solution to the Finsler MDS problem (\cref{eq:fmds}) using the canonical Randers space $\mathcal{X} = \mathbb{R}^{m+1}$ with $\omega\neq 0$ is
    given by the traditional MDS embedding in a
    $m$-dimensional
    hyperplane orthogonal to $\omega$.
\end{theorem}

A proof is provided in \cref{sec: proof of canonical randers generalises riemann with extra dim}. 
\Cref{th: canonical randers generalises riemann with extra dim} implies that the canonical Randers space can thus be seen as an extension of the Euclidean space, adding an extra dimension to encode asymmetry.
Within hyperplanes orthogonal to this direction, the metric remains Euclidean. When there is no asymmetry, Finsler MDS preserves the Euclidean embedding in the $m$-dimensional hyperplane.

\section{Solving Finsler Multi-Dimensional Scaling}
\label{sec: solving FMDS}

We first solve the Finsler MDS problem by revisiting the SMACOF algorithm for the canonical Randers metric (\cref{subsec:fsmacof}).
In addition, we show how to adapt the Finsler MDS problem to modern deep learning for node embedding and link prediction in digraphs (\cref{subsec:finsler_embedded_learning}).
These two methods show the simplicity and versatility of Finsler MDS.

\subsection{Finsler SMACOF}\label{subsec:fsmacof}

The traditional MDS task in \cref{eq:MDS} is a non-linear least-squares problem that can be solved with the SMACOF algorithm \cite{borg2007modern}. We propose to extend this approach for solving the Finsler MDS task in \cref{eq:fmds}. We call this method the \emph{Finsler SMACOF algorithm}. 

Denoting $W$ the weight matrix with entries $w_{ij}$, we assume a symmetric weighting scheme $W=W^\top$, even if the data dissimilarities $\boldsymbol{D}$ are not. This is a natural assumption meaning that the $(i,j)$ and $(j,i)$ pairs are as important for the embedding.
We can then rewrite \cref{eq:fmds} as
\begin{align}
    \sigma_F^{2}(\boldsymbol{X}) &= \tr(\boldsymbol{X}^{\top}V\boldsymbol{X}) + \tr(\boldsymbol{X}^{\top}V\boldsymbol{X}\omega \omega^{\top})  \nonumber\\ 
    &\hspace{2em}+ 2 \tr((W\odot \boldsymbol{D}-W^\top \odot \boldsymbol{D}^\top)\mathbbm{1}_{m}\omega^{\top} \boldsymbol{X}^{\top}) \nonumber\\
    &\hspace{2em}- 2 \tr(\boldsymbol{X}^{\top}B(\boldsymbol{X})\boldsymbol{X}) + \tr(\boldsymbol{D}W\boldsymbol{D}^{\top}) \,,
    \label{eq:sigma2 traces}
\end{align}
where $\tr$ is the trace operator, $\odot$ is the Hadamard product, $\mathbbm{1}_{m}\in\mathbb{R}^m$ is the vector with all entries equal to $1$, and
\begin{align}
    \label{eq:V}
    V_{ij} &=
        \begin{cases}
            -w_{ij} & \text{ if } i \neq j, \\
            \sum_{k\neq i} w_{ik} & \text{ if } i = j.
        \end{cases}
        \\
    \label{eq:B}
    B_{ij}(\boldsymbol{X}) &=     
    \begin{cases}
         - w_{ij}\frac{\boldsymbol{D}_{i,j}}{d_{F^C}(x_i, x_j)}& \text{if } i \neq j
         \\
         -\sum_{k\neq i}B_{ik}(\boldsymbol{X})& \text{if } i = j.
    \end{cases}
\end{align}

See \cref{sec:reformulation of fmds} for a detailed derivation.
Denote $\vect(A)$ the vectorisation of matrix $A$ by stacking its columns, $A^\dagger$ the Moore-Penrose pseudo-inverse of $A$, and $\otimes$ the Kronecker product. Let $K\in\mathbb{R}^{mN\times mN}$ and $C\in\mathbb{R}^{N\times m}$ be 
$K = (I_m + \omega\omega^\top) \otimes V $ and 
$C = (W\odot \boldsymbol{D}-W^\top \odot \boldsymbol{D}^\top)\mathbbm{1}_{m}\omega^{\top}$.
By deriving first-order optimality conditions, 
$\sigma_F$ can be iteratively minimised. 
Denoting $\boldsymbol{X}^{(k)}$ the $k$-th 
iterate of the embedding,
we have the following. 

\begin{proposition}[Finsler SMACOF]\label{prop:FSMACOF}
    A local minimum of the Finsler stress 
    with symmetric weights is reached by iterating
    $$ \vect\big(\boldsymbol{X}^{(k+1)}\big) = \\ K^{\dagger} \vect\Big(B\big(\boldsymbol{X}^{(k)}\big)\boldsymbol{X}^{(k)} - C \Big).$$
\end{proposition}
We refer the reader to \cref{sec:proof of fsmacof} for a proof.
Note that when $\omega \equiv 0$, we retrieve the iterates of the traditional SMACOF algorithm \cite{borg2007modern}: 
$\boldsymbol{X}^{(k+1)} = V^\dagger B\big(\boldsymbol{X}^{(k)}\big)\boldsymbol{X}^{(k)}$.

\subsection{Finsler Representation Learning}
\label{subsec:finsler_embedded_learning}

To scale our approach, we present a Finsler representation learning method leveraging modern deep learning. This approach enables the canonical Randers embedding to be learned directly from data through neural networks, allowing applications to large digraph datasets \cite{jure2014snap}. 
Rather than training on $\sigma_F$, we adopt 
it as an evaluation metric, optimizing our model with alternative objectives while still maintaining effective embeddings.
Such replacements of the stress with alternative losses are standard in the field.

To bridge Finsler MDS with Finsler representation learning for graph tasks, we introduce the concept of a \emph{digraph} (directed graph) and its relationship to an asymmetrical pairwise distance matrix.
Let $\mathcal{G}= (\mathcal{V}, \mathcal{E}, \mathbf{A})$ be a digraph, where $\mathcal{V}$ is the set of $\abs{\mathcal{V}} = N$ nodes, $\mathcal{E}
\subset
\mathcal{V} \times \mathcal{V}$ represents directed edges, and $\mathbf{A} \in \mathbb{R}^{N \times N}$ is the weighted adjacency matrix. 
The directed graph 
distance
$d_{\mathcal{G}}$
is the length of shortest paths on $\mathcal{G}$. 
This results in an asymmetrical pairwise distance matrix $\mathbf{D}$, where $\mathbf{D}_{i,j} = d_{\mathcal{G}}(i, j)$ for nodes $i, j \in \mathcal{V}$, inherently satisfying $\mathbf{D}^\top \neq \mathbf{D}$ due to the directed nature of $\mathcal{G}$. 
 It aligns with our asymmetrical setup in Finsler MDS (see \cref{sec:Finsler Multi-Dimensional Scaling}). 
Given a digraph $\mathcal{G}$, we aim to represent each node $i \in \mathcal{V}$ by a $m$-dimensional Finsler embedding $x_i\in\mathbb{R}^m$ such that the stress function in  Eq.~\eqref{eq:fmds} is minimised and the learned embedding supports various graph generalisations,  such as link prediction. 

To learn the canonical Randers embedding, we use the Fermi-Dirac decoder \cite{krioukov2010hyperbolic, nickel2017poincare} to model
probability scores of directed edges, given by 
\begin{equation}
    p\big( ( i, j ) \in \mathcal{E} | \bX \big) = \frac{1}{1 + \exp\left(
    \tfrac{d_{F^C}^2(x_i, x_j) - r}{t} \right) },
\end{equation}
with $d_{F^C}$ is the canonical Randers distance, and $r, t >0$ are hyperparameters. 
We train the Finsler embeddings on the weighted cross-entropy loss with negative sampling \cite{goldberg2014word2vec}.

\section{Experimental Results}
\label{sec:applications}

Our new problem and its resolution open new perspectives, as we demonstrate with two types of applications.
The first is data visualisation via low dimensional embeddings (\cref{subsec:data_visualisation}).
The second is the representation learning
for node embedding and link prediction for directed graphs (\cref{sec:digraph_embedding}), which are prevalent tasks in graph theory.
The implementation details, including experimental setup, dataset details, and hyperparameters, along with additional results can be found in
\cref{sec: additional details and res experiments}.

\subsection{Data Visualisation}\label{subsec:data_visualisation}

We demonstrate not only how
our method is capable of handling asymmetric input dissimilarities, but also how the canonical Randers embedding space provides an intuitive asymmetric embedding space suitable for visualisations. We show on several types of synthetic data the strength of our method for many visualisation applications.

The drift component of the canonical Randers embedding space $\omega$ is chosen to be along the last axis, i.e.\ $\vec{z}$ (resp. $\vec{y}$) in 3D (resp. 2D).
The distance between two points going upwards is then greater than the downwards distance, leading to intuitive visualisations. Results are shown using the Finsler SMACOF algorithm (\cref{subsec:fsmacof}) on data dissimilarities computed from geodesic distances via Dijkstra's algorithm \cite{dijkstra1959note}. Full details are provided in \cref{sec: visualisation exp implementation considerations}.

\begin{figure}[ht]
  \centering
   \includegraphics[width=\columnwidth]{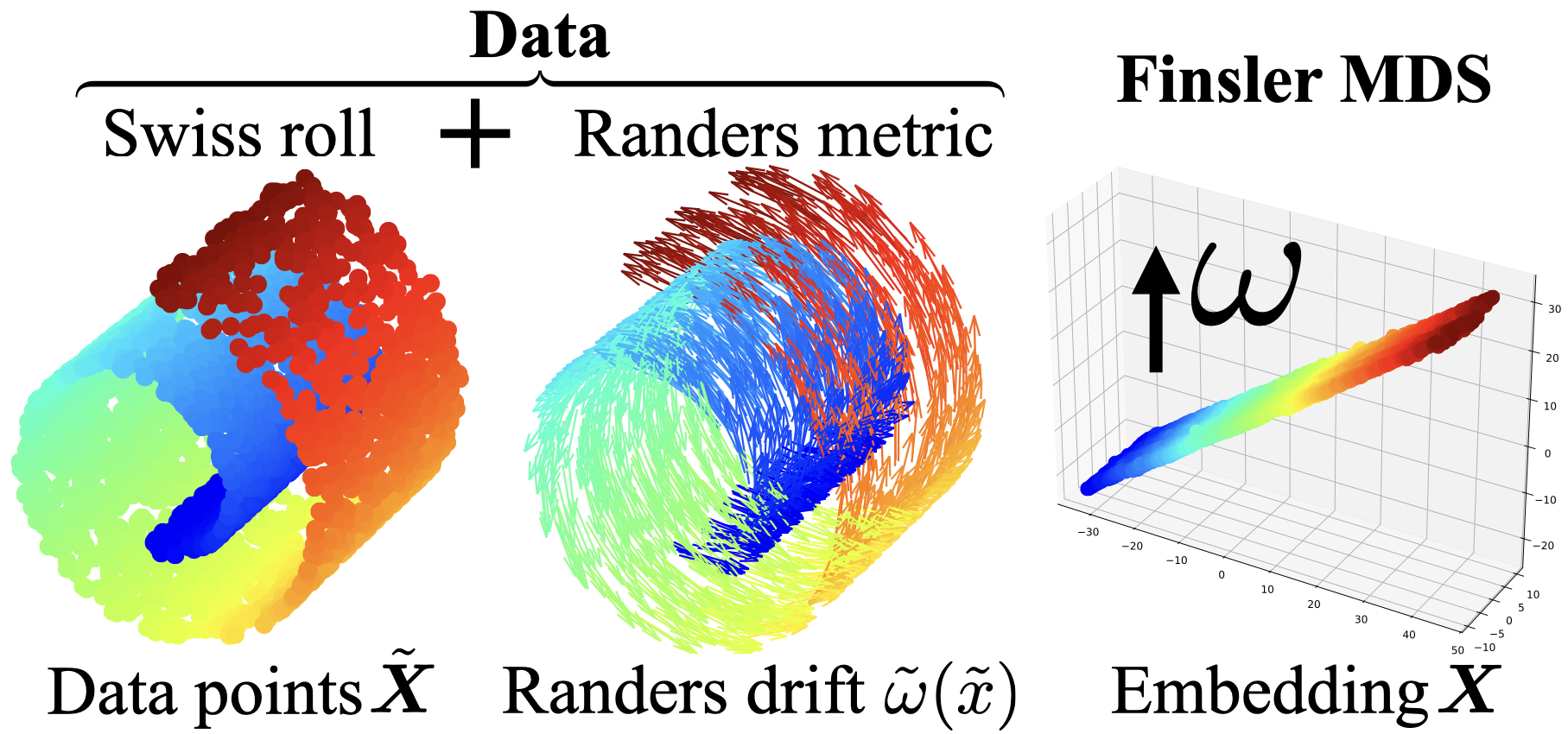}
   \caption{Flattening the Swiss roll (left) equipped with a Randers metric (middle) in the 3D canonical Randers space with Finsler MDS (right).
   The plotted arrows on the manifold are the linear drift components $\tilde{\omega}$ of the Randers metric. 
   While Finsler MDS manages to flatten the Swiss roll, it preserves the asymmetry along $\omega$, 
   e.g.\ the height difference between blue and red points implies asymmetric distances between them.
   }
   \label{fig: swiss role single alpha no isomap}
   \vspace{-1em}
\end{figure}

\hfill \break 
\noindent \textbf{Asymmetric Manifold Flattening.}
We flatten the Swiss roll $\tilde{\mathcal{X}}\in\mathbb{R}^3$, a reference shape in Manifold learning. To make it asymmetric, we equip $\tilde{\mathcal{X}}$ with the Randers metric $\tilde{F}^{\tilde{\alpha}}$ given by $\tilde{F}_{\tilde{x}}^{\tilde{\alpha}}(u) = \lVert u\rVert_2 + \tilde{\alpha} \tilde{\omega}^\top u$, with chosen 
$\tilde{\omega}\in \mathcal{T}_{\tilde{x}}\tilde{\mathcal{X}}$ having
$\lVert\tilde{\omega}\rVert_2=1$ and hyperparameter $\tilde{\alpha}<1$ controlling the input amount of asymmetry on $\tilde{\mathcal{X}}$.
We plot in \cref{fig: swiss role single alpha no isomap} the Finsler Smacof embedding to the canonical Randers space $\mathcal{X}=\mathbb{R}^3$, which is well-suited for embedding two-dimensional asymmetric manifolds (see \cref{th: canonical randers generalises riemann with extra dim}). 
For more embeddings with different levels of input asymmetry $\tilde{\alpha}$, and in particular an illustration of \cref{th: canonical randers generalises riemann with extra dim} when data is symmetric $\tilde{\alpha}=0$, see \cref{sec: additional results visualisation exp}.
Our Finsler embedding demonstrates both the intuitive flat structure of the Swiss roll $\tilde{\mathcal{X}}$ and its asymmetry between points. 
In fact, our visualisation is significantly clearer than when artificially 
plotting arrows on top of the manifold $\tilde{\mathcal{X}}$, as the asymmetry can be difficult to discern for numerous samples or
low magnitudes of the Randers drift component.
To the best of our knowledge, our method is the only metric approach capable of naturally revealing and preserving 
these asymmetric
structures, thanks to Finsler geometry.

\begin{figure}[ht]
  \centering
   \includegraphics[width=\columnwidth]{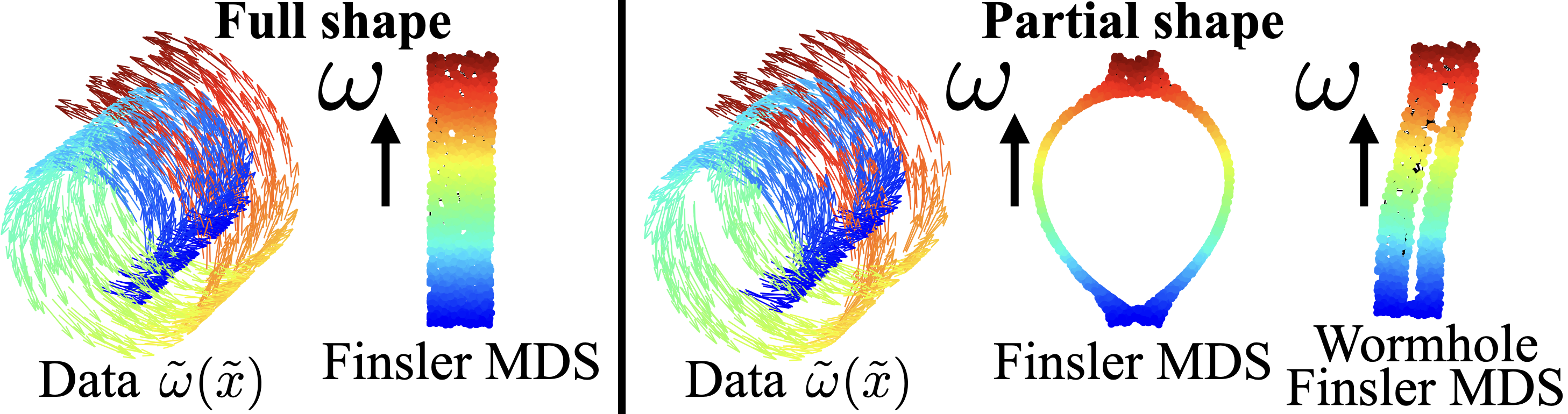}
   \caption{Swiss roll embeddings to the canonical Randers space $\mathbb{R}^2$. Our Finsler MDS can generalise current SOTA approaches for providing embeddings robust to missing parts.
   Note that Finsler MDS accurately embeds the Swiss roll to $\mathbb{R}^2$, but Isomap would not be able to accurately embed the symmetric version to $\mathbb{R}^1$.
   }
   \label{fig: swiss roll hole just randers}
   \vspace{-1em}
\end{figure}

\hfill \break
\noindent \textbf{Robustness to Holes.}
Most Manifold learning approaches suffer from topological changes
due to missing parts.
To increase robustness, the weights $w_{i,j}$ of the stress can be tuned to filter out perturbed data dissimilarities, by focusing on local \cite{schwartz2019intrinsic} or non-local criteria \cite{rosman2010nonlinear,bracha2024wormhole}.
Our Finsler MDS is compatible with such approaches.
We propose a Finsler wormhole criterion, generalising the SOTA non-heuristic criterion for symmetric distances \cite{bracha2024wormhole}, and prove that it recovers consistent geodesic pairs on a Finsler manifold in \cref{sec: wormhole Finsler MDS}. 
We showcase in \cref{fig: swiss roll hole just randers} the strength of the Finsler wormhole criterion to provide robust embeddings preserving the geometry of the underlying Finsler manifold.

\begin{figure}[ht]
  \centering
    \includegraphics[width=\columnwidth]{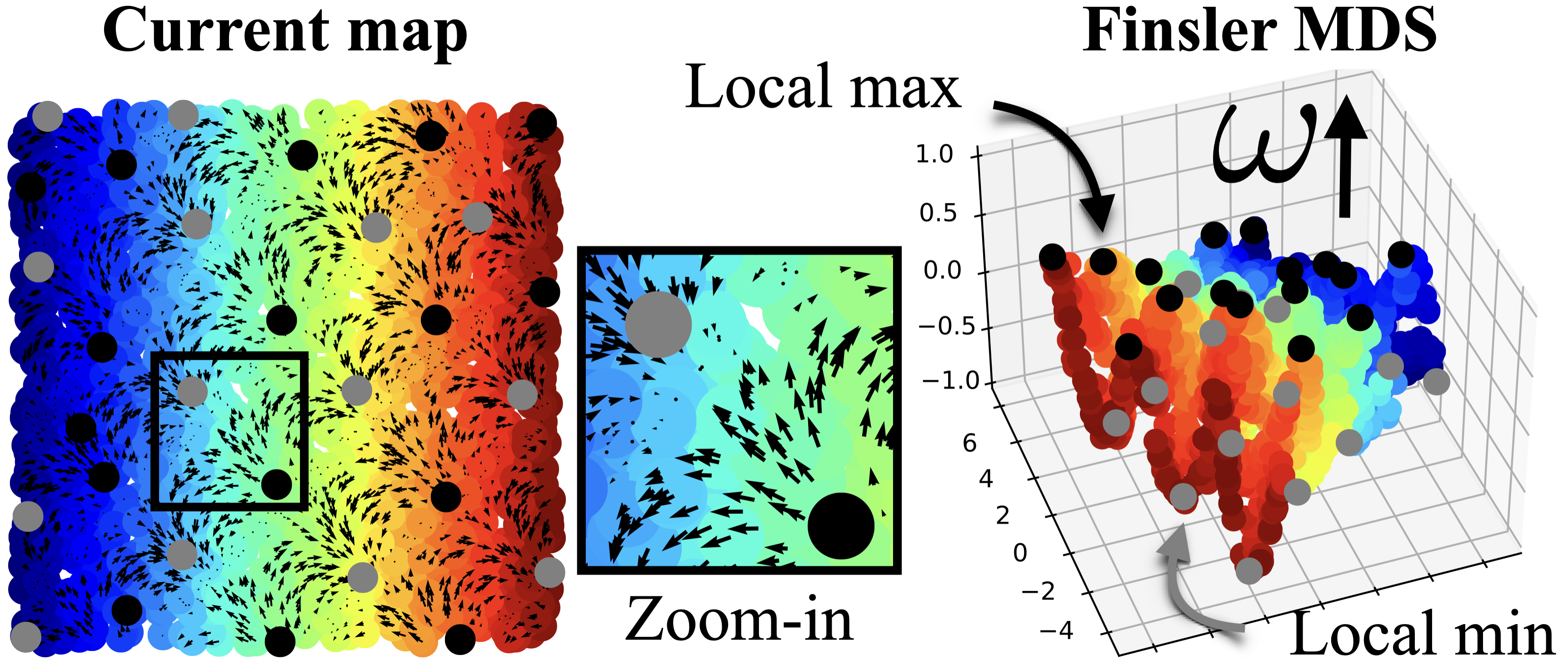}
    \caption{
    Unflattening a current map (left) by Finsler MDS embedding to the 3D canonical Randers space (right). Dissimilarities are shortest-time distances
    given the current. 
    The 3D map reveals the asymmetry, where timewise distances are easy to read based on the measurements on the straight Euclidean 3D line between points, and with local maxima (resp. minima), plotted in black (resp. gray), corresponding to source (resp. sink) points.
   }
   \label{fig: current maps sea}
   \vspace{-1em}
\end{figure}

\hfill \break
\noindent 
\textbf{Unflattening Current Maps.}
We aim to visualise current maps such as surface currents in a sea. 
The data points lie on the 2D plane $\tilde{\mathcal{X}}\subset\mathbb{R}^2$, 
with current at their location $\tilde{v}(\tilde{x})\in \mathcal{T}_{\tilde{x}}\tilde{\mathcal{X}} = \mathbb{R}^2$.
The 2D map displays the correct spatial flatness of the surface, and it can be used to measure spatial distances. 
However, due to complex non-uniform currents $\tilde{v}$, the time-shortest paths for a boat with a fixed speed motor will be particularly complex and we cannot measure on the 2D map the time of these journeys.
In fact, these times are given by the Zermelo metric $\tilde{F}$, which is Randers, that we approximate here as $\tilde{F}_x(u) = \lVert u\rVert_2 - \tilde{v}(\tilde{x})^\top u$
(see \cref{sec: current vs randers}).
We propose to embed the data into the 3D canonical Randers space $\mathcal{X}=\mathbb{R}^3$, yielding a 3D map where height levels represent asymmetry in shortest-time travel (see 
\cref{fig: current maps sea} and \cref{sec: additional results visualisation exp}).
Shortest travel time is then easily measurable (see \cref{prop:randers_distance}) 
from Euclidean segments.
Additionally, extracting local extrema from the embedding space yields the source and sink points of the current. 
The curved Finsler embedding is thus intuitive.

\begin{figure}[ht]
  \centering
    \includegraphics[width=\columnwidth]{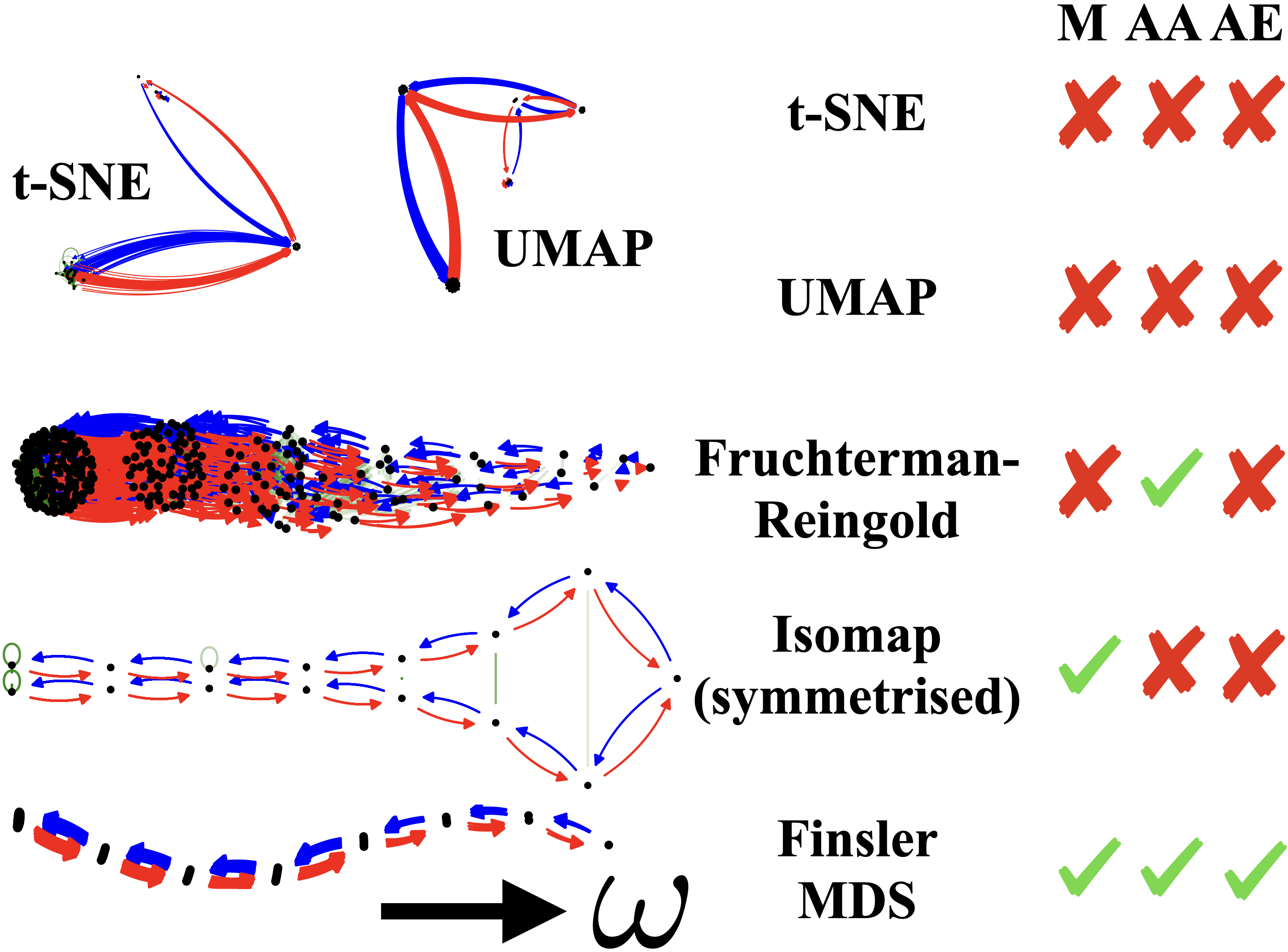}
   \caption{
    Manually rotated embeddings of modified directed binary trees of various depth, with small symmetric edges between equally deep nodes.
    The reference works, t-SNE \cite{van2008visualizing} and UMAP \cite{mcinnes2018umap} fail to provide a meaningful embedding.
    The 
    Fruchterman-Reingold \cite{fruchterman1991graph} algorithm
    fails to clearly reveal the hierarchy, 
    unlike
    Isomap \cite{schwartz1989numerical,tenenbaum2000global}. The latter heavily distorts
    distances between equally deep nodes, due to node collapse to two modes and large separations at the top of the structure.
    Neither method provides the direction of the hierarchy.
    In contrast, the Finsler MDS embedding clearly reveals the hierarchy and its direction by construction while highly preserving all distances. 
    The symbols mean: M -- Metric, AA -- Asymmetric algorithm, and AE -- Asymmetric Embedding.
   }
   \label{fig: binary tree one}
   \vspace{-1em}
\end{figure}

\hfill \break 
\noindent \textbf{Revealing Graph Hierarchies.}
Directed graphs are common and asymmetric, yet lack any embedding.
In this experiment, we aim to reveal the node hierarchy by embedding directed graphs into a 2D canonical Randers space.
Nodes embedded higher in $\vec{y}$ have greater importance than those embedded below.
Our synthetic graph is a modified directed binary tree. 
Children are close from their parents but not vice versa,  and small symmetric edges have been added between nodes at the same depth. 
This simulates a proximity graph in a pyramidal company between managers (parents) and employees (children). 
The Finsler embedding (see \cref{fig: binary tree one}) naturally reveals the graph hierarchy between nodes, with nodes at different depths embedded to different heights. 
In contrast, the traditional symmetrised MDS embedding collapses and the proximity between nodes at the same depth is not consistent throughout the graph.

\subsection{Digraph Embedding and Link Prediction}
\label{sec:digraph_embedding}
We demonstrate the power of the proposed Finsler representation learning (\cref{subsec:finsler_embedded_learning}) for digraph embedding and 
link prediction
tasks \cite{ou2016asymmetric, perrault2011directed}. 
Various approaches have been proposed to preserve the asymmetry information in graph data \cite{zhou2004semi, zhou2005learning, chung2006diameter, fanuel2017magnetic, zhang2021magnet, koke2023holonets, rossi2024edge}, capturing essential properties of directed structures. These methods explore ways to encode the directionality inherent in directed graphs, often through adjustments in learning algorithms or spectral properties that reflect asymmetrical relationships. Nevertheless, to the best of our knowledge, none of these approaches have considered Finsler geometry for digraphs. 
%

\begin{table*}[t]
\Huge
\centering
\resizebox{\textwidth}{!}{%
\begin{tabular}{lccccccccccccccccccccccccccr}
\toprule
 & & & \multicolumn{23}{c}{Dimensionality}\\
 \cmidrule{4-26} 
     & & & \multicolumn{5}{c}{2} & &\multicolumn{5}{c}{5} & &\multicolumn{5}{c}{10} & &\multicolumn{5}{c}{50} \\
     \cmidrule{4-8} \cmidrule{10-14} \cmidrule{16-20} \cmidrule{22-26}
     & & & \multicolumn{2}{c}{Distortion $(\downarrow)$} & & \multicolumn{2}{c}{MAP $(\uparrow)$} & & \multicolumn{2}{c}{Distortion $(\downarrow)$} & & \multicolumn{2}{c}{MAP $(\uparrow)$} & & \multicolumn{2}{c}{Distortion $(\downarrow)$} & & \multicolumn{2}{c}{MAP $(\uparrow)$} & & \multicolumn{2}{c}{Distortion $(\downarrow)$} & & \multicolumn{2}{c}{MAP $(\uparrow)$} & 
     \\
     \cmidrule{4-5} \cmidrule{7-8} \cmidrule{10-11} \cmidrule{13-14} \cmidrule{16-17} \cmidrule{19-20} \cmidrule{22-23} \cmidrule{25-26}
     & & & Euclidean & Finsler && Euclidean & Finsler 
     & & Euclidean & Finsler && Euclidean & Finsler 
     & & Euclidean & Finsler && Euclidean & Finsler
     & & Euclidean & Finsler && Euclidean & Finsler \\
\midrule
& Cora & & 0.426 & \textbf{0.127} && 0.574 & \textbf{0.894} 
& & 0.414 & \textbf{0.106} && 0.612 & \textbf{0.922} 
& & 0.322 & \textbf{0.093} && 0.677 & \textbf{0.931} 
& & 0.218 & \textbf{0.091} && 0.831 & \textbf{0.949} \\
\midrule
& Citeseer & & 0.386 & \textbf{0.194} && 0.394 & \textbf{0.691} 
& & 0.320 & \textbf{0.172} && 0.495 & \textbf{0.835} 
& & 0.305 & \textbf{0.083} && 0.681 & \textbf{0.958} 
& & 0.277 & \textbf{0.051} && 0.740 & \textbf{0.976} \\
\midrule
& Gr-QC & & 0.247 & \textbf{0.012} && 0.289 & \textbf{0.986} 
& & 0.213 & \textbf{0.011} && 0.384 & \textbf{0.994} 
& & 0.209 & \textbf{0.010} && 0.424 & \textbf{0.999} 
& & 0.192 & \textbf{0.010} && 0.487 & \textbf{0.999} \\
\midrule
& Chameleon & & 0.483 & \textbf{0.261} && 0.488 & \textbf{0.798} 
& & 0.450 & \textbf{0.129} && 0.502 & \textbf{0.849} 
& & 0.397 & \textbf{0.116} && 0.594 & \textbf{0.903} 
& & 0.376 & \textbf{0.102} && 0.628 & \textbf{0.924} \\
\midrule
& Squirrel & & 0.592 & \textbf{0.226} && 0.395 & \textbf{0.812} 
& & 0.514 & \textbf{0.201} && 0.465 & \textbf{0.837} 
& & 0.423 & \textbf{0.139} && 0.541 & \textbf{0.994} 
& & 0.388 & \textbf{0.130} && 0.607 & \textbf{0.996} \\
\midrule
& Arxiv-Year & & 0.489 & \textbf{0.242} && 0.321 & \textbf{0.742} 
& & 0.408 & \textbf{0.195} && 0.392 & \textbf{0.853} 
& & 0.346 & \textbf{0.136} && 0.415 & \textbf{0.927} 
& & 0.329 & \textbf{0.115} && 0.431 & \textbf{0.954} \\
\bottomrule
\end{tabular} 
}
\caption{Distortion and Mean Average Precision (MAP) for asymmetric distance reconstruction for digraph data (the better in bold). 
We obtain consistently, and by a significant margin, less distortion and better accuracy when embedding in our Finsler space than in the Euclidean space. In particular, Finsler embeddings lead to better performance across all datasets and all embedding dimensions.
} 
\label{tab: digraph_embedding}
\end{table*}

\hfill \break
\noindent \textbf{Digraph Embedding for Asymmetric Distance Reconstruction.}
To assess the quality of learned Finsler embeddings, we first compare embedding results on the full graph, i.e.\ all nodes, using either the canonical Randers distance or the Euclidean distance in the embedding space $\mathbb{R}^m$ of various dimensions $m\in\{2,5,10,50\}$.
The reconstruction errors, analysed with respect to
$m$,
provide an indication of the strength of Finsler embeddings. 
To measure the performance of the learned embedding, we use two common metrics: (i) distortion, defined as the normalised stress 
$\bar{\sigma}_L(\boldsymbol{X}) = \begin{psmallmatrix} N \\ 2 \end{psmallmatrix}^{-1}\sigma_L(\boldsymbol{X})$, where with $L=F^C$ (resp. $E$) for the canonical Finsler (resp. Euclidean) embedding space, and weights $w_{i,j} = \boldsymbol{D}_{i,j}^{-1}$,
and (ii) Mean Average Precision (MAP) score.
We embed six digraphs: Cora \cite{sen2008collective}, Citeseer \cite{yang2016revisiting},  Gr-QC \cite{leskovec2007graph}, Chameleon \cite{rozemberczki2021multi}, Squirrel \cite{rozemberczki2021multi}, and Arxiv-Year \cite{hu2020open}. 
For details on these datasets, refer to \cref{subseec:digraph_embedding_and_link_prediction}.
We show the reconstruction performance in
\cref{tab: digraph_embedding}. 
Our results suggest that Finsler embeddings excel in representing digraphs. Indeed, compared to Euclidean representations, the associated distortions are consistently smaller and the accuracy higher, for all digraphs and by a significant margin. Moreover, Finsler embeddings show consistent performance across embedding dimensions $m$, providing a stable representation even in low dimensions, where they capture the asymmetrical structure effectively.
This makes them a practical choice for digraph embedding.

\begin{table*}[t]
\LARGE
\centering
\resizebox{\textwidth}{!}{%
\begin{tabular}{lccccccccccccccccccccccccccccr}
\toprule
 & \multicolumn{6}{c}{Existence Prediction}   & & \multicolumn{6}{c}{Direction Prediction}   \\ 
\cmidrule{2-7} \cmidrule{9-14} 
& Cora & Citeseer & Gr-Qc & Chameleon & Squirrel & Arxiv-Year &  & Cora & Citeseer & Gr-Qc & Chameleon & Squirrel& Arxiv-Year\\
\midrule
NERD  & 82.7$\pm$0.8 & 79.2$\pm$0.3 & 73.9$\pm$0.6  & 81.5$\pm$0.7  & 72.5$\pm$0.8  & 54.2$\pm$1.3 &  & 90.6$\pm$0.6 & 81.3$\pm$0.7  & 79.2$\pm$0.8  & 84.2$\pm$0.4  & 78.2$\pm$0.5  & 59.4$\pm$1.6    \\ 
DiGCN & 81.5$\pm$0.4 & 81.9$\pm$1.7  & 76.2$\pm$1.3  & 78.2$\pm$1.9  & 73.4$\pm$1.5  & 59.6$\pm$2.1 &  & 90.2$\pm$1.5 & 87.2$\pm$1.8 & 82.5$\pm$1.6  & 85.0$\pm$1.2  & 80.1$\pm$1.4  & 64.5$\pm$1.2 \\ 
MagNet & 84.2$\pm$0.9 & 87.5$\pm$0.9 & 79.4$\pm$1.5  & 83.2$\pm$0.7  & \textbf{82.6}$\pm$1.0  & \underline{65.8$\pm$1.3} &  & 93.2$\pm$0.4 & 94.9$\pm$0.8  & 87.3$\pm$0.6  & 89.7$\pm$1.0  & \underline{85.2$\pm$0.8}  & \underline{70.2$\pm$2.4}  \\
DiGAE & 81.8$\pm$0.5 & 85.3$\pm$2.8 & 74.8$\pm$2.1  & 75.4$\pm$1.8  & 71.9$\pm$0.9  & 60.2$\pm$1.6 &  & 87.6$\pm$0.7 & 80.6$\pm$2.7  & 83.2$\pm$1.8  & 83.6$\pm$1.4  & 81.3$\pm$1.5  & 62.5$\pm$1.3     \\
ODIN &  89.1$\pm$0.6 & 85.0$\pm$1.9 &  82.5$\pm$1.4  & 84.6$\pm$1.3 & 78.7$\pm$1.4  & 63.5$\pm$1.8 &  & \textbf{95.3}$\pm$0.5 & 93.2$\pm$0.7  & 87.0$\pm$0.9   & 90.3$\pm$0.8  & 83.7$\pm$1.2  & 68.9$\pm$1.4  \\ 
DUPLEX & \underline{95.0$\pm$0.2} & \textbf{97.2}$\pm$0.6 & \underline{83.1$\pm$0.3} & \underline{87.2$\pm$0.4} & \underline{82.5$\pm$0.6} & 64.7$\pm$1.2 & & \textbf{95.3}$\pm$0.3 & \textbf{97.9}$\pm$0.1 & \underline{90.7$\pm$0.1} & \underline{92.0$\pm$0.4} & 85.0$\pm$0.9 & 69.6$\pm$1.1\\
\midrule
Ours & \textbf{95.2}$\pm$0.3 & \underline{94.6$\pm$1.0} &  \textbf{84.9}$\pm$0.6  & \textbf{89.3}$\pm$0.8  & 82.3$\pm$0.3  & \textbf{68.6}$\pm$1.4 &  & \underline{95.1$\pm$0.2} & \underline{97.2$\pm$0.3}  & \textbf{92.6}$\pm$0.5  & \textbf{93.2}$\pm$0.4  & \textbf{89.0}$\pm$0.3  & \textbf{74.4}$\pm$1.8  \\ 
\bottomrule
\end{tabular} 
}
\caption{Area Under the ROC Curve (ROC AUC) for the link prediction tasks (the highest in bold and the second highest underlined).
For each task, our Finsler embedding largely outperforms existing methods for nearly all datasets, and is on par for the remaining dataset.
} 
\label{tab: link_prediction}
\vspace{-0.2em}
\end{table*}

\hfill \break
\noindent \textbf{Link Prediction.}
To assess the generalisation capability of the proposed Finsler embedding, we conduct two link prediction experiments: existence and direction prediction.
We compare our method with NERD \cite{khosla2020node}, DiGCN \cite{tong2020digraph}, MagNet \cite{zhang2021magnet}, DiGAE \cite{kollias2022directed}, ODIN \cite{yoo2023disentangling}, and DUPLEX \cite{ke2024duplex}. 
We conduct experiments on the same digraphs as in the embedding task. 
Here, we split the graph datasets into 80\% for training, 15\% for testing, and 5\% for validation. 
Performance is assessed by measuring the Area Under the ROC Curve (AUC).
We measure the link prediction quality by computing the average performance and standard deviation over 10 random splits.
We present in 
\cref{tab: link_prediction} the link prediction performance of our Finsler embedding and competing baselines. 
For the link existence prediction task, our method surpasses by a large margin competing methods on all datasets except Squirrel, where our method is competitive with the best one. In the link direction prediction task, the Finsler approach achieves the best performance by a significant margin on all datasets except Cora, where it is on par with the best one.
For both experiments, our method is either far superior to the baselines by several percentage points, or is comparable to the best baseline within $0.3\%$.
These results demonstrate the 
superior
performance and generalisation of our Finsler method across both tasks.

\section{Conclusion}

We propose a novel framework for multi-dimensional scaling, called \textit{Finsler MDS}. 
It extends traditional MDS by embedding data in a Finsler manifold rather than a Riemannian one.
By doing so, we naturally accommodate asymmetry in data dissimilarities, which is a notorious blind spot in traditional methods.
Inspired by the Euclidean space, we design a simple canonical Finsler space in which to embed asymmetric data, enabling intuitive and computationally efficient representations.
Through both theoretical analysis and empirical validation, we demonstrate that Finsler MDS achieves accurate embeddings and preserves essential structures in diverse datasets where asymmetric relationships are critical.
We experimentally showcase the ability of Finsler MDS to provide novel insights in a wide range of visualisation applications. 
We also demonstrate its contribution to the field of graph representation learning providing superior embeddings for node embedding and link prediction, 
areas where traditional methods fall short.
Additionally, we illustrate the flexibility of the Finsler MDS framework by incorporating modern tools from both the traditional MDS,
and the deep learning communities.
We hope that this framework will inspire further research in asymmetric manifold learning and the use of Finsler geometry in computer vision applications, 
opening pathways for richer data representations and deeper insights into the complex structures underlying high-dimensional asymmetric data, beyond the mainstream Riemannian geometry.

\noindent \textbf{Limitations and Future Work.} 
First, due to asymmetry, providing a closed-form solution of the Finsler MDS embedding is non-trivial. Indeed, the extension of the kernel-based solution for simple cases in Euclidean geometry is not straightforward. 
Instead, we rely on iterative descent strategies to find a local minimum.
Although Finsler MDS extends traditional MDS methods and their Euclidean space by allowing a asymmetric geometrical representation, it similarly leads to increasing distortion the more the data manifold departs from the chosen metric structure. 
Let us also mention that in the canonical Randers space, all the asymmetry is encoded in a single dimension. 
However, there might be applications where two independent semantic concepts create asymmetry, and then other Finsler metrics 
might be preferable
to generate disentangled asymmetric embeddings.
Finally, let us highlight that the encoding of asymmetry via Finsler geometry extends beyond our dimensionality reduction and embedding applications. Nevertheless, it has been largely understudied for practical applications in computer vision and deserves further attention.

\section*{Acknowledgments}
This work was supported by the ERC Advanced Grant "SIMULACRON" (agreement \#884679), by the Israel Science Foundation (ISF) (grant No. 3864/21), by ADRI – Advanced Defense Research Institute – Technion and by the Rosenblum Fund.


{
    \small
    \bibliographystyle{ieeenat_fullname}
    \bibliography{main}

\begin{thebibliography}{117}
\providecommand{\natexlab}[1]{#1}
\providecommand{\url}[1]{\texttt{#1}}
\expandafter\ifx\csname urlstyle\endcsname\relax
  \providecommand{\doi}[1]{doi: #1}\else
  \providecommand{\doi}{doi: \begingroup \urlstyle{rm}\Url}\fi

\bibitem[Aflalo and Kimmel(2013)]{aflalo2013spectral}
Yonathan Aflalo and Ron Kimmel.
\newblock Spectral multidimensional scaling.
\newblock \emph{Proceedings of the National Academy of Sciences}, 110\penalty0 (45):\penalty0 18052--18057, 2013.

\bibitem[Akiba et~al.(2019)Akiba, Sano, Yanase, Ohta, and Koyama]{akiba2019optuna}
Takuya Akiba, Shotaro Sano, Toshihiko Yanase, Takeru Ohta, and Masanori Koyama.
\newblock Optuna: A next-generation hyperparameter optimization framework.
\newblock In \emph{Proceedings of the 25th ACM SIGKDD international conference on knowledge discovery \& data mining}, pages 2623--2631, 2019.

\bibitem[Asanov(2012)]{asanov2012finsler}
Gennadi{\u{i}}~Semenovich Asanov.
\newblock \emph{{F}insler geometry, relativity and gauge theories}.
\newblock Springer Science \& Business Media, 2012.

\bibitem[Bao et~al.(2012)Bao, Chern, and Shen]{bao2012introduction}
David Bao, S-S Chern, and Zhongmin Shen.
\newblock \emph{An introduction to {R}iemann-{F}insler geometry}.
\newblock Springer Science \& Business Media, 2012.

\bibitem[Belkin and Niyogi(2001)]{belkin2001laplacian}
Mikhail Belkin and Partha Niyogi.
\newblock Laplacian eigenmaps and spectral techniques for embedding and clustering.
\newblock \emph{Advances in neural information processing systems}, 14, 2001.

\bibitem[Belkin and Niyogi(2003)]{belkin2003laplacian}
Mikhail Belkin and Partha Niyogi.
\newblock Laplacian eigenmaps for dimensionality reduction and data representation.
\newblock \emph{Neural computation}, 15\penalty0 (6):\penalty0 1373--1396, 2003.

\bibitem[Bidabadi and Sedaghat(2010)]{bidabadi2010geometric}
B Bidabadi and M Sedaghat.
\newblock Geometric modeling of {D}ubins airplane movement and its metric.
\newblock \emph{AUT Journal of Electrical Engineering}, 42\penalty0 (1):\penalty0 9--16, 2010.

\bibitem[Bonnabel(2013)]{bonnabel2013stochastic}
Silvere Bonnabel.
\newblock Stochastic gradient descent on riemannian manifolds.
\newblock \emph{IEEE Transactions on Automatic Control}, 58\penalty0 (9):\penalty0 2217--2229, 2013.

\bibitem[Bonnans et~al.(2022)Bonnans, Bonnet, and Mirebeau]{bonnans2022linear}
J~Frederic Bonnans, Guillaume Bonnet, and Jean-Marie Mirebeau.
\newblock A linear finite-difference scheme for approximating {R}anders distances on cartesian grids.
\newblock \emph{ESAIM: Control, Optimisation and Calculus of Variations}, 28:\penalty0 45, 2022.

\bibitem[Borg and Groenen(2007)]{borg2007modern}
Ingwer Borg and Patrick~JF Groenen.
\newblock \emph{Modern multidimensional scaling: Theory and applications}.
\newblock Springer Science \& Business Media, 2007.

\bibitem[Bracha et~al.(2024{\natexlab{a}})Bracha, Dag{\`e}s, and Kimmel]{bracha2024unsupervised}
Amit Bracha, Thomas Dag{\`e}s, and Ron Kimmel.
\newblock On unsupervised partial shape correspondence.
\newblock In \emph{Proceedings of the Asian Conference on Computer Vision}, pages 4488--4504, 2024{\natexlab{a}}.

\bibitem[Bracha et~al.(2024{\natexlab{b}})Bracha, Dag{\`e}s, and Kimmel]{bracha2024wormhole}
Amit Bracha, Thomas Dag{\`e}s, and Ron Kimmel.
\newblock Wormhole loss for partial shape matching.
\newblock \emph{Advances in Neural Information Processing Systems}, 2024{\natexlab{b}}.

\bibitem[Brand(2005)]{brand2005nonrigid}
Matthew Brand.
\newblock Nonrigid embeddings for dimensionality reduction.
\newblock In \emph{European Conference on Machine Learning}, pages 47--59. Springer, 2005.

\bibitem[Bronstein et~al.(2006{\natexlab{a}})Bronstein, Bronstein, and Kimmel]{bronstein2006efficient}
Alexander~M Bronstein, Michael~M Bronstein, and Ron Kimmel.
\newblock Efficient computation of isometry-invariant distances between surfaces.
\newblock \emph{SIAM Journal on Scientific Computing}, 28\penalty0 (5):\penalty0 1812--1836, 2006{\natexlab{a}}.

\bibitem[Bronstein et~al.(2006{\natexlab{b}})Bronstein, Bronstein, and Kimmel]{bronstein2006generalised}
Alexander~M Bronstein, Michael~M Bronstein, and Ron Kimmel.
\newblock Generalized multidimensional scaling: a framework for isometry-invariant partial surface matching.
\newblock \emph{Proceedings of the National Academy of Sciences}, 103\penalty0 (5):\penalty0 1168--1172, 2006{\natexlab{b}}.

\bibitem[Bronstein et~al.(2008)Bronstein, Bronstein, and Kimmel]{bronstein2008numerical}
Alexander~M Bronstein, Michael~M Bronstein, and Ron Kimmel.
\newblock \emph{Numerical geometry of non-rigid shapes}.
\newblock Springer Science \& Business Media, 2008.

\bibitem[Caponio et~al.(2014)Caponio, Javaloyes, and S{\'a}nchez]{caponio2014wind}
Erasmo Caponio, Miguel~Angel Javaloyes, and Miguel S{\'a}nchez.
\newblock Wind {F}inslerian structures: from {Z}ermelo's navigation to the causality of spacetimes.
\newblock \emph{arXiv preprint arXiv:1407.5494}, 2014.

\bibitem[Cartan(1933)]{cartan1933espaces}
Elie Cartan.
\newblock Sur les espaces de {F}insler.
\newblock \emph{Comptes rendus de l'Académie des Sciences}, 196:\penalty0 582–--586, 1933.

\bibitem[Chen et~al.(2015)Chen, Mirebeau, and Cohen]{chen2015global}
Da Chen, Jean-Marie Mirebeau, and Laurent~D. Cohen.
\newblock Global minimum for curvature penalized minimal path method.
\newblock In \emph{Proc. of the British Machine Vision Conference (BMVC)}, pages 86.1--86.12, 2015.

\bibitem[Chen et~al.(2016{\natexlab{a}})Chen, Mirebeau, and Cohen]{chen2016geoevo}
Da Chen, Jean-Marie Mirebeau, and Laurent~D. Cohen.
\newblock {F}insler geodesics evolution model for region based active contours.
\newblock In \emph{Proc. of the British Machine Vision Conference (BMVC)}, 2016{\natexlab{a}}.

\bibitem[Chen et~al.(2016{\natexlab{b}})Chen, Mirebeau, and Cohen]{chen2016minpath}
Da Chen, Jean-Marie Mirebeau, and Laurent~D. Cohen.
\newblock A new {F}insler minimal path model with curvature penalization for image segmentation and closed contour detection.
\newblock In \emph{IEEE Conference on Computer Vision and Pattern Recognition (CVPR)}, pages 355--363, 2016{\natexlab{b}}.

\bibitem[Chen et~al.(2017)Chen, Mirebeau, and Cohen]{chen2017elastica}
Da Chen, Jean-Marie Mirebeau, and Laurent~D. Cohen.
\newblock Global minimum for a {F}insler elastica minimal path approach.
\newblock pages 458--483, 2017.

\bibitem[Chern(1996)]{chern1996finsler}
Shiing-Shen Chern.
\newblock {F}insler geometry is just {R}iemannian geometry without the quadratic equation.
\newblock \emph{Notices of the American Mathematical Society}, 43\penalty0 (9):\penalty0 959--963, 1996.

\bibitem[Chino(1978)]{chino1978graphical}
Naohito Chino.
\newblock A graphical technique for representing the asymmetric relationships between n objects.
\newblock \emph{Behaviormetrika}, 5:\penalty0 23--40, 1978.

\bibitem[Chino(2011)]{chino2011asymmetric}
Naohito Chino.
\newblock Asymmetric multidimensional scaling.
\newblock \emph{Journal of the Institute for Psychological and Physical Science}, 3\penalty0 (1):\penalty0 101--107, 2011.

\bibitem[Chung(2006)]{chung2006diameter}
Fan Chung.
\newblock The diameter and laplacian eigenvalues of directed graphs.
\newblock \emph{the electronic journal of combinatorics}, pages N4--N4, 2006.

\bibitem[Clayton(2015)]{clayton2015finsler}
John~D Clayton.
\newblock On {F}insler geometry and applications in mechanics: review and new perspectives.
\newblock \emph{Advances in Mathematical Physics}, 2015\penalty0 (1):\penalty0 828475, 2015.

\bibitem[Coifman and Lafon(2006)]{coifman2006diffusion}
Ronald~R Coifman and St{\'e}phane Lafon.
\newblock Diffusion maps.
\newblock \emph{Applied and computational harmonic analysis}, 21\penalty0 (1):\penalty0 5--30, 2006.

\bibitem[Coifman et~al.(2005)Coifman, Lafon, Lee, Maggioni, Nadler, Warner, and Zucker]{coifman2005geometric}
Ronald~R Coifman, Stephane Lafon, Ann~B Lee, Mauro Maggioni, Boaz Nadler, Frederick Warner, and Steven~W Zucker.
\newblock Geometric diffusions as a tool for harmonic analysis and structure definition of data: {D}iffusion maps.
\newblock \emph{Proceedings of the national academy of sciences}, 102\penalty0 (21):\penalty0 7426--7431, 2005.

\bibitem[Constantine and Gower(1978)]{constantine1978graphical}
AG Constantine and John~C Gower.
\newblock Graphical representation of asymmetric matrices.
\newblock \emph{Journal of the Royal Statistical Society: Series C (Applied Statistics)}, 27\penalty0 (3):\penalty0 297--304, 1978.

\bibitem[Cox and Cox(2000)]{cox2000multidimensional}
Trevor~F Cox and Michael~AA Cox.
\newblock \emph{Multidimensional scaling}.
\newblock CRC press, 2000.

\bibitem[Dag{\`e}s et~al.(2024)Dag{\`e}s, Lindenbaum, and Bruckstein]{dages2024metric}
Thomas Dag{\`e}s, Michael Lindenbaum, and Alfred~M Bruckstein.
\newblock Metric convolutions: A unifying theory to adaptive convolutions.
\newblock \emph{arXiv preprint arXiv:2406.05400}, 2024.

\bibitem[De~Leeuw(1977)]{leeuw1977application}
Jan De~Leeuw.
\newblock Application of convex analysis to multidimensional scaling.
\newblock \emph{Recent developments in statistics}, pages 133--145, 1977.

\bibitem[de~Leeuw and Heiser(1982)]{leeuw1982theory}
Jan de Leeuw and Willem Heiser.
\newblock Theory of multidimensional scaling.
\newblock \emph{Handbook of statistics}, 2:\penalty0 285--316, 1982.

\bibitem[Dijkstra(1959)]{dijkstra1959note}
EW Dijkstra.
\newblock A note on two problems in connexion with graphs.
\newblock \emph{Numerische Mathematik}, 1\penalty0 (1):\penalty0 269--271, 1959.

\bibitem[Do~Carmo(2016)]{do2016differential}
Manfredo~P Do~Carmo.
\newblock \emph{Differential geometry of curves and surfaces: revised and updated second edition}.
\newblock Courier Dover Publications, 2016.

\bibitem[Donoho and Grimes(2003)]{donoho2003hessian}
David~L Donoho and Carrie Grimes.
\newblock Hessian eigenmaps: locally linear embedding techniques for high-dimensional data.
\newblock \emph{Proceedings of the National Academy of Sciences}, 100\penalty0 (10):\penalty0 5591--5596, 2003.

\bibitem[Fanuel et~al.(2017)Fanuel, Alaiz, and Suykens]{fanuel2017magnetic}
Micha{\"e}l Fanuel, Carlos~M Alaiz, and Johan~AK Suykens.
\newblock Magnetic eigenmaps for community detection in directed networks.
\newblock \emph{Physical Review E}, 95\penalty0 (2):\penalty0 022302, 2017.

\bibitem[Finsler(1918)]{finsler1918ueber}
Paul Finsler.
\newblock \emph{Über {K}urven und {F}l{\"a}chen in allgemeinen {R}{\"a}umen}.
\newblock Philosophische Fakultät, Georg-August-Univ., 1918.

\bibitem[Fruchterman and Reingold(1991)]{fruchterman1991graph}
Thomas~MJ Fruchterman and Edward~M Reingold.
\newblock Graph drawing by force-directed placement.
\newblock \emph{Software: Practice and experience}, 21\penalty0 (11):\penalty0 1129--1164, 1991.

\bibitem[Funke et~al.(2020)Funke, Guo, Lancic, and Antulov-Fantulin]{funke2020low}
Thorben Funke, Tian Guo, Alen Lancic, and Nino Antulov-Fantulin.
\newblock Low-dimensional statistical manifold embedding of directed graphs.
\newblock In \emph{8th International Conference on Learning Representations (ICLR 2020)}, pages 2018--2035. Curran, 2020.

\bibitem[Goldberg(2014)]{goldberg2014word2vec}
Yoav Goldberg.
\newblock Word2vec explained: deriving {M}ikolov et al.'s negative-sampling word-embedding method.
\newblock \emph{arXiv preprint arXiv:1402.3722}, 2014.

\bibitem[Gou et~al.(2023)Gou, Yuan, Xue, Du, Yu, Xia, and Zhang]{gou2023discriminative}
Jianping Gou, Xia Yuan, Ya Xue, Lan Du, Jiali Yu, Shuyin Xia, and Yi Zhang.
\newblock Discriminative and geometry-preserving adaptive graph embedding for dimensionality reduction.
\newblock \emph{Neural Networks}, 157:\penalty0 364--376, 2023.

\bibitem[Gower(1977)]{gower1977analysis}
John~C Gower.
\newblock The analysis of asymmetry and orthogonality.
\newblock \emph{Recent developments in statistics}, pages 109--123, 1977.

\bibitem[Groenen and van~de Velden(2016)]{groenen2016multidimensional}
Patrick~JF Groenen and Michel van~de Velden.
\newblock Multidimensional scaling by majorization: A review.
\newblock \emph{Journal of Statistical Software}, 73:\penalty0 1--26, 2016.

\bibitem[Hohmann et~al.(2019)Hohmann, Pfeifer, and Voicu]{hohmann2019finsler}
Manuel Hohmann, Christian Pfeifer, and Nicoleta Voicu.
\newblock {F}insler gravity action from variational completion.
\newblock \emph{Physical Review D}, 100\penalty0 (6):\penalty0 064035, 2019.

\bibitem[Hotelling(1933)]{hotelling1933analysis}
Harold Hotelling.
\newblock Analysis of a complex of statistical variables into principal components.
\newblock \emph{Journal of educational psychology}, 24\penalty0 (6):\penalty0 417, 1933.

\bibitem[Hu et~al.(2020)Hu, Fey, Zitnik, Dong, Ren, Liu, Catasta, and Leskovec]{hu2020open}
Weihua Hu, Matthias Fey, Marinka Zitnik, Yuxiao Dong, Hongyu Ren, Bowen Liu, Michele Catasta, and Jure Leskovec.
\newblock Open graph benchmark: Datasets for machine learning on graphs.
\newblock \emph{Advances in neural information processing systems}, 33:\penalty0 22118--22133, 2020.

\bibitem[Javaloyes and S{\'a}nchez(2011)]{javaloyes2011definition}
Miguel~Angel Javaloyes and Miguel S{\'a}nchez.
\newblock On the definition and examples of {F}insler metrics.
\newblock \emph{arXiv preprint arXiv:1111.5066}, 2011.

\bibitem[Jure(2014)]{jure2014snap}
Leskovec Jure.
\newblock Snap datasets: Stanford large network dataset collection.
\newblock \emph{Retrieved December 2021 from http://snap. stanford. edu/data}, 2014.

\bibitem[Ke et~al.(2024)Ke, Yu, Li, and Zhang]{ke2024duplex}
Zhaoru Ke, Hang Yu, Jianguo Li, and Haipeng Zhang.
\newblock Duplex: Dual gat for complex embedding of directed graphs.
\newblock \emph{arXiv preprint arXiv:2406.05391}, 2024.

\bibitem[Khosla et~al.(2020)Khosla, Leonhardt, Nejdl, and Anand]{khosla2020node}
Megha Khosla, Jurek Leonhardt, Wolfgang Nejdl, and Avishek Anand.
\newblock Node representation learning for directed graphs.
\newblock In \emph{Machine Learning and Knowledge Discovery in Databases: European Conference, ECML PKDD 2019, W{\"u}rzburg, Germany, September 16--20, 2019, Proceedings, Part I}, pages 395--411. Springer, 2020.

\bibitem[Koke and Cremers(2024)]{koke2023holonets}
Christian Koke and Daniel Cremers.
\newblock Holonets: Spectral convolutions do extend to directed graphs.
\newblock In \emph{The Twelfth International Conference on Learning Representations}, 2024.

\bibitem[Kollias et~al.(2022)Kollias, Kalantzis, Id{\'e}, Lozano, and Abe]{kollias2022directed}
Georgios Kollias, Vasileios Kalantzis, Tsuyoshi Id{\'e}, Aur{\'e}lie Lozano, and Naoki Abe.
\newblock Directed graph auto-encoders.
\newblock In \emph{Proceedings of the AAAI conference on artificial intelligence}, pages 7211--7219, 2022.

\bibitem[Krioukov et~al.(2010)Krioukov, Papadopoulos, Kitsak, Vahdat, and Bogun{\'a}]{krioukov2010hyperbolic}
Dmitri Krioukov, Fragkiskos Papadopoulos, Maksim Kitsak, Amin Vahdat, and Mari{\'a}n Bogun{\'a}.
\newblock Hyperbolic geometry of complex networks.
\newblock \emph{Physical Review E—Statistical, Nonlinear, and Soft Matter Physics}, 82\penalty0 (3):\penalty0 036106, 2010.

\bibitem[Kropina(1959)]{kropina1959projective}
VK Kropina.
\newblock On projective {F}insler spaces with a metric of some special form.
\newblock \emph{Naucn. Dokl Vyss. Skoly. Fiz-Mat. Mauki}, 2:\penalty0 38--42, 1959.

\bibitem[Kruskal(1964)]{kruskal1964multidimensional}
Joseph~B Kruskal.
\newblock Multidimensional scaling by optimizing goodness of fit to a nonmetric hypothesis.
\newblock \emph{Psychometrika}, 29\penalty0 (1):\penalty0 1--27, 1964.

\bibitem[Leskovec et~al.(2007)Leskovec, Kleinberg, and Faloutsos]{leskovec2007graph}
Jure Leskovec, Jon Kleinberg, and Christos Faloutsos.
\newblock Graph evolution: Densification and shrinking diameters.
\newblock \emph{ACM transactions on Knowledge Discovery from Data (TKDD)}, 1\penalty0 (1):\penalty0 2--es, 2007.

\bibitem[Lin et~al.(2023)Lin, Coifman, Mishne, and Talmon]{lin2023hyperbolic}
Ya-Wei~Eileen Lin, Ronald~R Coifman, Gal Mishne, and Ronen Talmon.
\newblock Hyperbolic diffusion embedding and distance for hierarchical representation learning.
\newblock In \emph{International Conference on Machine Learning}, pages 21003--21025. PMLR, 2023.

\bibitem[Lin et~al.(2024)Lin, Coifman, Mishne, and Talmon]{lin2024tree}
Ya-Wei~Eileen Lin, Ronald~R Coifman, Gal Mishne, and Ronen Talmon.
\newblock Tree-wasserstein distance for high dimensional data with a latent feature hierarchy.
\newblock \emph{arXiv preprint arXiv:2410.21107}, 2024.

\bibitem[Lopez et~al.(2021)Lopez, Pozzetti, Trettel, Strube, and Wienhard]{lopez2021symmetric}
Federico Lopez, Beatrice Pozzetti, Steve Trettel, Michael Strube, and Anna Wienhard.
\newblock Symmetric spaces for graph embeddings: A {F}insler-{R}iemannian approach.
\newblock In \emph{International Conference on Machine Learning}, pages 7090--7101. PMLR, 2021.

\bibitem[Mart{\'\i}n-Merino and Mu{\~n}oz(2005)]{martin2005visualizing}
Manuel Mart{\'\i}n-Merino and Alberto Mu{\~n}oz.
\newblock Visualizing asymmetric proximities with {SOM} and {MDS} models.
\newblock \emph{Neurocomputing}, 63:\penalty0 171--192, 2005.

\bibitem[Matsumoto(1972)]{matsumoto1972c}
Makoto Matsumoto.
\newblock On {C}-reducible {F}insler spaces.
\newblock \emph{Tensor, NS}, 24:\penalty0 29--37, 1972.

\bibitem[Matsumoto(1989)]{matsumoto1989slope}
Makoto Matsumoto.
\newblock A slope of a mountain is a {F}insler surface with respect to a time measure.
\newblock \emph{Journal of Mathematics of Kyoto University}, 29\penalty0 (1):\penalty0 17--25, 1989.

\bibitem[McInnes et~al.(2018)McInnes, Healy, and Melville]{mcinnes2018umap}
Leland McInnes, John Healy, and James Melville.
\newblock Umap: Uniform manifold approximation and projection for dimension reduction.
\newblock \emph{arXiv preprint arXiv:1802.03426}, 2018.

\bibitem[Mirebeau(2014)]{mirebeau2014efficient}
Jean-Marie Mirebeau.
\newblock Efficient fast marching with {F}insler metrics.
\newblock \emph{Numerische mathematik}, 126\penalty0 (3):\penalty0 515--557, 2014.

\bibitem[Nickel and Kiela(2017)]{nickel2017poincare}
Maximillian Nickel and Douwe Kiela.
\newblock Poincar{\'e} embeddings for learning hierarchical representations.
\newblock \emph{Advances in Neural Information Processing Systems}, 30, 2017.

\bibitem[Ohta and Sturm(2009)]{ohta2009heat}
Shin-ichi Ohta and Karl-Theodor Sturm.
\newblock Heat flow on {F}insler manifolds.
\newblock \emph{Communications on Pure and Applied Mathematics: A Journal Issued by the Courant Institute of Mathematical Sciences}, 62\penalty0 (10):\penalty0 1386--1433, 2009.

\bibitem[Okada and Imaizumi(1987)]{okada1987nonmetric}
Akinori Okada and Tadashi Imaizumi.
\newblock Nonmetric multidimensional scaling of asymmetric proximities.
\newblock \emph{Behaviormetrika}, 14\penalty0 (21):\penalty0 81--96, 1987.

\bibitem[Okada(2012)]{okada2012bayesian}
Kensuke Okada.
\newblock A bayesian approach to asymmetric multidimensional scaling.
\newblock \emph{Behaviormetrika}, 39:\penalty0 49--62, 2012.

\bibitem[Olszewski(2024)]{olszewski2024asymmetric}
Dominik Olszewski.
\newblock Asymmetric {I}somap for dimensionality reduction and data visualization.
\newblock In \emph{International Conference on Artificial Neural Networks}, pages 102--115. Springer, 2024.

\bibitem[Ou et~al.(2016)Ou, Cui, Pei, Zhang, and Zhu]{ou2016asymmetric}
Mingdong Ou, Peng Cui, Jian Pei, Ziwei Zhang, and Wenwu Zhu.
\newblock Asymmetric transitivity preserving graph embedding.
\newblock In \emph{Proceedings of the 22nd ACM SIGKDD international conference on Knowledge discovery and data mining}, pages 1105--1114, 2016.

\bibitem[Paszke(2019)]{paszke2019pytorch}
A Paszke.
\newblock Pytorch: An imperative style, high-performance deep learning library.
\newblock \emph{arXiv preprint arXiv:1912.01703}, 2019.

\bibitem[Pearson(1901)]{pearson1901pca}
Karl Pearson.
\newblock On lines and planes of closest fit to systems of points in space.
\newblock \emph{The London, Edinburgh, and Dublin philosophical magazine and journal of science}, 2\penalty0 (11):\penalty0 559--572, 1901.

\bibitem[Perrault-Joncas and Meila(2011)]{perrault2011directed}
Dominique Perrault-Joncas and Marina Meila.
\newblock Directed graph embedding: an algorithm based on continuous limits of laplacian-type operators.
\newblock \emph{Advances in neural information processing systems}, 24, 2011.

\bibitem[Pfeifer and Wohlfarth(2012)]{pfeifer2012finsler}
Christian Pfeifer and Mattias~NR Wohlfarth.
\newblock {F}insler geometric extension of einstein gravity.
\newblock \emph{Physical Review D—Particles, Fields, Gravitation, and Cosmology}, 85\penalty0 (6):\penalty0 064009, 2012.

\bibitem[Randers(1941)]{randers1941asymmetrical}
Gunnar Randers.
\newblock On an asymmetrical metric in the four-space of general relativity.
\newblock \emph{Physical Review}, 59\penalty0 (2):\penalty0 195, 1941.

\bibitem[Ratliff et~al.(2021)Ratliff, Van~Wyk, Xie, Li, and Asif~Rana]{ratliff2021}
Nathan~D. Ratliff, Karl Van~Wyk, Mandy Xie, Anqi Li, and Muhammad Asif~Rana.
\newblock Generalized nonlinear and {F}insler geometry for robotics.
\newblock In \emph{IEEE International Conference on Robotics and Automation (ICRA)}, pages 10206--10212, 2021.

\bibitem[Riemann(1854)]{riemann1854hypotheses}
Bernhard Riemann.
\newblock On the hypotheses which lie at the foundations of geometry.
\newblock \emph{A source book in mathematics}, 2:\penalty0 411--425, 1854.

\bibitem[Rosman et~al.(2010)Rosman, Bronstein, Bronstein, and Kimmel]{rosman2010nonlinear}
Guy Rosman, Michael~M Bronstein, Alexander~M Bronstein, and Ron Kimmel.
\newblock Nonlinear dimensionality reduction by topologically constrained isometric embedding.
\newblock \emph{International Journal of Computer Vision}, 89\penalty0 (1):\penalty0 56--68, 2010.

\bibitem[Rossi et~al.(2024)Rossi, Charpentier, Di~Giovanni, Frasca, G{\"u}nnemann, and Bronstein]{rossi2024edge}
Emanuele Rossi, Bertrand Charpentier, Francesco Di~Giovanni, Fabrizio Frasca, Stephan G{\"u}nnemann, and Michael~M Bronstein.
\newblock Edge directionality improves learning on heterophilic graphs.
\newblock In \emph{Learning on Graphs Conference}, pages 25--1. PMLR, 2024.

\bibitem[Roweis and Saul(2000)]{roweis2000nonlinear}
Sam~T Roweis and Lawrence~K Saul.
\newblock Nonlinear dimensionality reduction by locally linear embedding.
\newblock \emph{science}, 290\penalty0 (5500):\penalty0 2323--2326, 2000.

\bibitem[Rozemberczki et~al.(2021)Rozemberczki, Allen, and Sarkar]{rozemberczki2021multi}
Benedek Rozemberczki, Carl Allen, and Rik Sarkar.
\newblock Multi-scale attributed node embedding.
\newblock \emph{Journal of Complex Networks}, 9\penalty0 (2):\penalty0 cnab014, 2021.

\bibitem[Saad and Schultz(1986)]{saad1986gmres}
Youcef Saad and Martin~H Schultz.
\newblock {GMRES}: A generalized minimal residual algorithm for solving nonsymmetric linear systems.
\newblock \emph{SIAM Journal on scientific and statistical computing}, 7\penalty0 (3):\penalty0 856--869, 1986.

\bibitem[Saeed et~al.(2018)Saeed, Nam, Haq, and Muhammad~Saqib]{saeed2018survey}
Nasir Saeed, Haewoon Nam, Mian Imtiaz~Ul Haq, and Dost~Bhatti Muhammad~Saqib.
\newblock A survey on multidimensional scaling.
\newblock \emph{ACM Computing Surveys (CSUR)}, 51\penalty0 (3):\penalty0 1--25, 2018.

\bibitem[Sch{\"o}lkopf et~al.(1997)Sch{\"o}lkopf, Smola, and M{\"u}ller]{scholkopf1997kernel}
Bernhard Sch{\"o}lkopf, Alexander Smola, and Klaus-Robert M{\"u}ller.
\newblock Kernel principal component analysis.
\newblock In \emph{International conference on artificial neural networks}, pages 583--588. Springer, 1997.

\bibitem[Sch{\"o}lkopf et~al.(1998)Sch{\"o}lkopf, Smola, and M{\"u}ller]{scholkopf1998nonlinear}
Bernhard Sch{\"o}lkopf, Alexander Smola, and Klaus-Robert M{\"u}ller.
\newblock Nonlinear component analysis as a kernel eigenvalue problem.
\newblock \emph{Neural computation}, 10\penalty0 (5):\penalty0 1299--1319, 1998.

\bibitem[Schwartz and Talmon(2019)]{schwartz2019intrinsic}
Ariel Schwartz and Ronen Talmon.
\newblock Intrinsic isometric manifold learning with application to localization.
\newblock \emph{SIAM Journal on Imaging Sciences}, 12\penalty0 (3):\penalty0 1347--1391, 2019.

\bibitem[Schwartz et~al.(1989)Schwartz, Shaw, and Wolfson]{schwartz1989numerical}
E.L. Schwartz, A. Shaw, and E. Wolfson.
\newblock A numerical solution to the generalized mapmaker's problem: {F}lattening nonconvex polyhedral surfaces.
\newblock \emph{IEEE Transactions on Pattern Analysis and Machine Intelligence}, 11\penalty0 (9):\penalty0 1005--1008, 1989.

\bibitem[Sen et~al.(2008)Sen, Namata, Bilgic, Getoor, Galligher, and Eliassi-Rad]{sen2008collective}
Prithviraj Sen, Galileo Namata, Mustafa Bilgic, Lise Getoor, Brian Galligher, and Tina Eliassi-Rad.
\newblock Collective classification in network data.
\newblock \emph{AI magazine}, 29\penalty0 (3):\penalty0 93--93, 2008.

\bibitem[Shen and Shen(2016)]{shen2016introduction}
Yi-Bing Shen and Zhongmin Shen.
\newblock \emph{Introduction to modern {F}insler geometry}.
\newblock World Scientific Publishing Company, 2016.

\bibitem[Shen(2001)]{shen2001lectures}
Zhongmin Shen.
\newblock \emph{Lectures on {F}insler geometry}.
\newblock World Scientific, 2001.

\bibitem[Shen(2003)]{shen2003finsler}
Zhongmin Shen.
\newblock {F}insler metrics with {K}= 0 and {S}= 0.
\newblock \emph{Canadian Journal of Mathematics}, 55\penalty0 (1):\penalty0 112--132, 2003.

\bibitem[Shepard(1962)]{shepard1962analysis}
Roger~N Shepard.
\newblock The analysis of proximities: multidimensional scaling with an unknown distance function. i.
\newblock \emph{Psychometrika}, 27\penalty0 (2):\penalty0 125--140, 1962.

\bibitem[Suzuki et~al.(2019)Suzuki, Takahama, and Onoda]{suzuki2019hyperbolic}
Ryota Suzuki, Ryusuke Takahama, and Shun Onoda.
\newblock Hyperbolic disk embeddings for directed acyclic graphs.
\newblock In \emph{International Conference on Machine Learning}, pages 6066--6075. PMLR, 2019.

\bibitem[Tanioka and Yadohisa(2018)]{tanioka2018asymmetric}
Kensuke Tanioka and Hiroshi Yadohisa.
\newblock Asymmetric mds with categorical external information based on radius model.
\newblock \emph{Procedia Computer Science}, 140:\penalty0 284--291, 2018.

\bibitem[Tenenbaum et~al.(2000)Tenenbaum, Silva, and Langford]{tenenbaum2000global}
Joshua~B Tenenbaum, Vin~de Silva, and John~C Langford.
\newblock A global geometric framework for nonlinear dimensionality reduction.
\newblock \emph{science}, 290\penalty0 (5500):\penalty0 2319--2323, 2000.

\bibitem[Tong et~al.(2020)Tong, Liang, Sun, Li, Rosenblum, and Lim]{tong2020digraph}
Zekun Tong, Yuxuan Liang, Changsheng Sun, Xinke Li, David Rosenblum, and Andrew Lim.
\newblock Digraph inception convolutional networks.
\newblock \emph{Advances in neural information processing systems}, 33:\penalty0 17907--17918, 2020.

\bibitem[Tversky(1977)]{tversky1977features}
Amos Tversky.
\newblock Features of similarity.
\newblock \emph{Psychological review}, 84\penalty0 (4):\penalty0 327, 1977.

\bibitem[Vacaru et~al.(2005)Vacaru, Stavrinos, Gaburov, and Gon{\c{t}}a]{vacaru2005clifford}
S Vacaru, P Stavrinos, E Gaburov, and D Gon{\c{t}}a.
\newblock {C}lifford and {R}iemann-{F}insler structures in geometric mechanics and gravity.
\newblock \emph{arXiv preprint gr-qc/0508023}, 2005.

\bibitem[Vacaru(2011)]{vacaru2011principles}
Sergiu~I Vacaru.
\newblock Principles of {E}instein-{F}insler gravity and cosmology.
\newblock In \emph{Journal of Physics: Conference Series}, page 012069. IOP Publishing, 2011.

\bibitem[Van~der Maaten and Hinton(2008)]{van2008visualizing}
Laurens Van~der Maaten and Geoffrey Hinton.
\newblock Visualizing data using t-{SNE}.
\newblock \emph{Journal of machine learning research}, 9\penalty0 (11), 2008.

\bibitem[Venna et~al.(2010)Venna, Peltonen, Nybo, Aidos, and Kaski]{venna2010information}
Jarkko Venna, Jaakko Peltonen, Kristian Nybo, Helena Aidos, and Samuel Kaski.
\newblock Information retrieval perspective to nonlinear dimensionality reduction for data visualization.
\newblock \emph{Journal of Machine Learning Research}, 11\penalty0 (2), 2010.

\bibitem[Weber et~al.(2024)Weber, Dag{\`e}s, Gao, and Cremers]{weber2024finsler}
Simon Weber, Thomas Dag{\`e}s, Maolin Gao, and Daniel Cremers.
\newblock {F}insler-{L}aplace-{B}eltrami operators with application to shape analysis.
\newblock In \emph{Proceedings of the IEEE/CVF Conference on Computer Vision and Pattern Recognition}, pages 3131--3140, 2024.

\bibitem[Weinberger et~al.(2005)Weinberger, Packer, and Saul]{weinberger2005nonlinear}
Kilian Weinberger, Benjamin Packer, and Lawrence Saul.
\newblock Nonlinear dimensionality reduction by semidefinite programming and kernel matrix factorization.
\newblock In \emph{International Workshop on Artificial Intelligence and Statistics}, pages 381--388. PMLR, 2005.

\bibitem[Weinberger and Saul(2006)]{weinberger2006unsupervised}
Kilian~Q Weinberger and Lawrence~K Saul.
\newblock Unsupervised learning of image manifolds by semidefinite programming.
\newblock \emph{International journal of computer vision}, 70:\penalty0 77--90, 2006.

\bibitem[Yajima and Nagahama(2009)]{yajima2009finsler}
Takahiro Yajima and Hiroyuki Nagahama.
\newblock {F}insler geometry of seismic ray path in anisotropic media.
\newblock \emph{Proceedings of the Royal Society A: Mathematical, Physical and Engineering Sciences}, 465\penalty0 (2106):\penalty0 1763--1777, 2009.

\bibitem[Yang et~al.(2019)Yang, Chai, Chen, and Cohen]{yang2019geodesic}
Fang Yang, Li Chai, Da Chen, and Laurent Cohen.
\newblock Geodesic via asymmetric heat diffusion based on {F}insler metric.
\newblock In \emph{Computer Vision--ACCV 2018: 14th Asian Conference on Computer Vision, Perth, Australia, December 2--6, 2018, Revised Selected Papers, Part V 14}, pages 371--386. Springer, 2019.

\bibitem[Yang et~al.(2016)Yang, Cohen, and Salakhudinov]{yang2016revisiting}
Zhilin Yang, William Cohen, and Ruslan Salakhudinov.
\newblock Revisiting semi-supervised learning with graph embeddings.
\newblock In \emph{International conference on machine learning}, pages 40--48. PMLR, 2016.

\bibitem[Yoo et~al.(2023)Yoo, Lee, Shin, and Kim]{yoo2023disentangling}
Hyunsik Yoo, Yeon-Chang Lee, Kijung Shin, and Sang-Wook Kim.
\newblock Disentangling degree-related biases and interest for out-of-distribution generalized directed network embedding.
\newblock In \emph{Proceedings of the ACM Web Conference 2023}, pages 231--239, 2023.

\bibitem[Young(1975)]{young1975asymmetric}
Forrest~W Young.
\newblock An asymmetric euclidean model for multi-process asymmetric data.
\newblock In \emph{US-Japan Seminar on MDS, San Diego, USA, 1975}, 1975.

\bibitem[Zermelo(1931)]{zermelo1931navigationsproblem}
Ernst Zermelo.
\newblock {\"U}ber das navigationsproblem bei ruhender oder ver{\"a}nderlicher windverteilung.
\newblock \emph{ZAMM-Journal of Applied Mathematics and Mechanics/Zeitschrift f{\"u}r Angewandte Mathematik und Mechanik}, 11\penalty0 (2):\penalty0 114--124, 1931.

\bibitem[Zhang et~al.(2021)Zhang, He, Brugnone, Perlmutter, and Hirn]{zhang2021magnet}
Xitong Zhang, Yixuan He, Nathan Brugnone, Michael Perlmutter, and Matthew Hirn.
\newblock Magnet: A neural network for directed graphs.
\newblock \emph{Advances in neural information processing systems}, 34:\penalty0 27003--27015, 2021.

\bibitem[Zhang and Wang(2006)]{zhang2006mlle}
Zhenyue Zhang and Jing Wang.
\newblock {MLLE}: {M}odified locally linear embedding using multiple weights.
\newblock \emph{Advances in neural information processing systems}, 19, 2006.

\bibitem[Zhang and Zha(2004)]{zhang2004principal}
Zhenyue Zhang and Hongyuan Zha.
\newblock Principal manifolds and nonlinear dimensionality reduction via tangent space alignment.
\newblock \emph{SIAM journal on scientific computing}, 26\penalty0 (1):\penalty0 313--338, 2004.

\bibitem[Zhou et~al.(2004)Zhou, Hofmann, and Sch{\"o}lkopf]{zhou2004semi}
Dengyong Zhou, Thomas Hofmann, and Bernhard Sch{\"o}lkopf.
\newblock Semi-supervised learning on directed graphs.
\newblock \emph{Advances in neural information processing systems}, 17, 2004.

\bibitem[Zhou et~al.(2005)Zhou, Huang, and Sch{\"o}lkopf]{zhou2005learning}
Dengyong Zhou, Jiayuan Huang, and Bernhard Sch{\"o}lkopf.
\newblock Learning from labeled and unlabeled data on a directed graph.
\newblock In \emph{Proceedings of the 22nd international conference on Machine learning}, pages 1036--1043, 2005.

\end{thebibliography}
}


\clearpage
\setcounter{page}{1}
\maketitlesupplementary
\appendix

\noindent 
This supplemental material is organized as follows:

\begin{itemize}

    \item 
    \textbf{\Cref{sec:proofs and derivations}} contains the proofs of \cref{th: canonical Randers space is flat,prop:randers_distance,th: canonical randers generalises riemann with extra dim,prop:FSMACOF} and the details of the derivations for the Finsler stress function (\cref{eq:sigma2 traces}).

    \item 
    \textbf{\Cref{sec: current vs randers,sec: wormhole Finsler MDS}} contain additional theoretical discussions. The former is dedicated to the link between current fields and Randers metrics, while the latter focuses on a generalisation of the Wormhole criterion to Finsler MDS to handle manifolds with missing parts.

    \item
    \textbf{\cref{sec: additional details and res experiments}} contains implementation details and additional experiments complementing the visualisation experiments in \cref{subsec:data_visualisation}
    and the digraph representation learning experiments in \cref{sec:digraph_embedding}.
    
\end{itemize}

\section{Proofs and Derivations}\label{sec:proofs and derivations}

\subsection{Proof of \cref{th: canonical Randers space is flat}}
\label{sec: proof of canonical Randers space is flat}

We here provide two proofs of this result. 
The first uses the Euler-Lagrange equation, a powerful and general tool in the calculus of variations. It can give some insights for generalisation to other metrics.
However, given the simplicity of the canonical Randers space, a quick and direct proof is also given.

\hfill \break \noindent \textbf{Euler-Lagrange.}
In calculus of variations, the Euler-Lagrange equation provides first order optimality necessary conditions on the solution of functionals involving functions $x(t)$ and their derivative $x'(t)$.

\begin{theorem}[Euler-Lagrange equation]
    If a functional of a smooth scalar function $x(t)$ is given by
    $\mathcal{L}(x) = \int_0^1 L(t,x(t), x'(t)) dt$, where $L$ is a positive smooth function, then the solution minimising the functional $\mathcal{L}$ satisfies the equation
    $$\frac{\partial \mathcal{L}}{\partial x} - \frac{d}{dt} \frac{\partial \mathcal{L}}{\partial x'} = 0.$$
\end{theorem}

Many generalisations of the Euler-Lagrange equations exist. In our case, when $x(t) = (x_1(t),\cdots, x_m(t))^\top\in\mathbb{R}^m$ is multi-dimensional, the Euler-Lagrange equation is duplicated for each output dimension. In other words, the minimum solution satisfies the set of equations
\begin{equation}
    \frac{\partial \mathcal{L}}{\partial x_i} - \frac{d}{dt} \frac{\partial \mathcal{L}}{\partial x_i'} = 0 \quad\forall i\in\{1,\cdots, m\}.
\end{equation}

The Euler-Lagrange equations can be used to derive shortest geodesic paths in our canonical Randers space. 
The length is a functional (\cref{eq:length}), that can be rewritten from a Lagrangian perspective as 
\begin{equation}
    \mathcal{L}_{F^C}(\gamma) = \int_0^1 L\big(t,\gamma(t),\gamma'(t)\big)dt,    
\end{equation}
where
\begin{equation}
    L\big(t,\gamma(t),\gamma'(t)\big) = F_{\gamma(t)}^C\big(\gamma'(t)\big) = \lVert \gamma'(t)\rVert_2 + \omega^\top \gamma'(t).
\end{equation}

Denoting $\gamma = (\gamma_1,\cdots,\gamma_m)$ and $\gamma' = (\gamma_1',\cdots,\gamma_m')$, the Euler-Lagrange equations for this functional are given for all $i\in\{1,\cdots,m\}$ by
\begin{equation}
    \frac{\partial L}{\partial \gamma_i} - \frac{d}{dt} \frac{\partial L}{\partial \gamma_i'} = 0.
\end{equation}

Since $L$ does not explicitly depend on $\gamma$, but only its derivative, we have that the Euler-Lagrange equations simplify to
\begin{align}
    0 &= \frac{d}{dt}\frac{\partial \mathcal{L}}{\partial \gamma_i'} \\
    &= \frac{d}{dt}\left(\frac{\gamma_i'(t)}{\lVert \gamma'(t)\rVert_2} + \omega_i\right).
\end{align}

In the canonical space, $\omega$ is a uniform vector field, as such its coordinates $\omega_i$ do not depend on $t$. Thus, $\tfrac{d}{dt}\omega_i = 0$. We then have, stacking the Euler-Lagrange equations into vector form, that
\begin{equation}
    \label{eq: d/dt gamma'/norm(gamma') = 0}
    \frac{d}{dt}\left(\frac{\gamma'(t)}{\lVert \gamma'(t)\rVert_2}\right) = 0.
\end{equation}

\Cref{eq: d/dt gamma'/norm(gamma') = 0} is the same as the one we would obtain if $\omega\equiv 0$, i.e. if the metric was Riemannian. It is well-known to describe the equation of a straight line. To see this, if we take $t=s$ to be the Euclidean arclength parametrisation, then $\lVert \gamma'(s)\rVert_2=1$ and then the Euler-Lagrange equation becomes $\frac{d}{ds}\gamma'(s)=0$, meaning that $\gamma'(s)$ is constant and thus $\gamma(s)$ is a straight Euclidean line. Shortest paths in the canonical Randers space are thus the straight segments as in the Euclidean space, making it a flat space.

\hfill \break \noindent \textbf{Calculation.} To better understand the particular structure of the canonical Randers space, we provide an alternative simple proof. Assume without loss of generality that $\omega = \alpha (0,\cdots, 0, 1)^\top$, and denote $\gamma(t) = (x_1(t),\cdots,x_m(t))^\top$. Then 
\begin{align}
     \mathcal{L}_{F^C}(\gamma) &= \int_0^1 F_{\gamma(t)}(\gamma'(t))dt\cr 
        &= \int_0^1 \left(\lVert \gamma'(t)\rVert_2 + \omega^\top \gamma'(t)\right) dt\cr 
        &= \int_0^1 \lVert \gamma'(t)\rVert_2 dt + \alpha \int_0^1 x_m' dt\cr
        &= \int_0^1 \lVert \gamma'(t)\rVert_2 dt + \alpha \int_{x_0}^{x_1} dx\cr
        &= \int_0^1 \lVert \gamma'(t)\rVert_2 dt + \alpha (x_1-x_0).
\end{align}
The right term is a constant not depending on the curve $\gamma$, whereas the left term is the usual functional giving the Euclidean length of the curve $\gamma$. Thus, the shortest path in the canonical Randers space is also the shortest path in the Euclidean space, which is given by the Euclidean segment $\gamma(t) = (1-t)x + t y$.

\subsection{Proof of \cref{prop:randers_distance}}
\label{sec: proof of randers canonical distance}

\hfill \break \noindent Although the shortest paths are the same in the canonical Randers space and in the Euclidean space, i.e.\ $\gamma_{x\to y}^{F^C}(t) = (1-t)x + ty$, their lengths are not the same as they depend on the direction of traversal. Since the metric is canonical, it does not depend on the position $\gamma_{x\to y}^{F^C}(t)$. Noticing that $\left(\gamma_{x\to y}^{F^C}\right)'(t) = y-x$, a direct calculation gives
\begin{align}
    d_{F^C}(x,y) &= \mathcal{L}_{F^C}\left(\gamma_{x\to y}^{F^C}\right)\nonumber\\
    &=\int_0^1 F^C_{\gamma_{x\to y}^{F^C}(t)}\bigg(\left(\gamma_{x\to y}^{F^C}\right)'(t)\bigg)dt\nonumber\\ 
    &= \int_0^1 \bigg(\left\lVert \left(\gamma_{x\to y}^{F^C}\right)'(t)\right\rVert_2 + \omega^\top \left(\gamma_{x\to y}^{F^C}\right)'(t)\bigg) dt\nonumber\\
    &= \int_0^1 \big(\lVert y-x\rVert_2 + \omega^\top (y-x)\big)dt\nonumber\\
    &=\lVert y-x\rVert_2 + \omega^\top (y-x). \qed
\end{align}

\subsection{Proof of \cref{th: canonical randers generalises riemann with extra dim}}
\label{sec: proof of canonical randers generalises riemann with extra dim}

By assumption, the data can be accurately embedded in the Euclidean space $\mathbb{R}^m$. Denote $\boldsymbol{X}\in\mathbb{R}^m$ this solution, with 
$d(x_i, x_j) = \boldsymbol{D}_{i,j}$ for all pairs $(i,j)$.
Consider now the Finsler MDS problem into the canonical Randers space of dimension $\mathbb{R}^{m+1}$.
Without loss of generality, we can assume that $\omega$ is along the last coordinate axis.
The embedding $\boldsymbol{Y} = [\boldsymbol{X}, 0]\in\mathbb{R}^{N\times (m+1)}$, which is the concatenation of the $m$-dimensional Euclidean embedding with a last $0$ coordinate is the minimal solution. Indeed, since the embedding  lies in a hyperplane orthogonal to $\omega$, we have
$d_F(x_i, x_j) = d_E(x_i, x_j)$ for all pairs $(i,j)$.
Since the Euclidean embedding is accurate, we have 
$d_{F^C}(x_i,x_j) = \boldsymbol{D}_{i,j}$ for all pairs $(i,j)$.
\qed

\subsection{Derivation of \cref{eq:sigma2 traces}}
\label{sec:reformulation of fmds}

Plugging into the Finsler stress (\cref{eq:fmds}) the canonical Randers distances between embedded points (\cref{eq:randers distance}), we have
\begin{align}
    \sigma^{2}(\boldsymbol{X}) &= \sum_{i,j} w_{ij} \lVert x_{j} - x_{i} \rVert_2^{2} \nonumber\\
    &\hspace{2em}+ 2 \sum_{i,j} w_{ij}\lVert x_{j} - x_{i} \rVert_2\omega^{\top} (x_{j}-x_{i})^{\top} \nonumber\\
    &\hspace{2em}+ \sum_{i,j} w_{ij}(x_{j}-x_{i})\omega \omega^{\top} (x_{j}-x_{i})^{\top} \nonumber\\
    &\hspace{2em}- 2 \sum_{i,j}w_{ij}\boldsymbol{D}_{ij}\lVert x_j - x_i \rVert_2 \nonumber\\
    &\hspace{2em}- 2 \sum_{i,j} w_{ij}\boldsymbol{D}_{ij}\omega^{\top} (x_j - x_i)^{\top} + \sum_{i,j} w_{ij}\boldsymbol{D}_{ij}^{2}.
\end{align}
As $w_{ij} = w_{ji}$, the second summation term vanishes
\begin{align}
       \sigma^{2}(\boldsymbol{X}) &= \sum_{i,j} w_{ij} \lVert x_{j} - x_{i} \rVert_2^{2} \nonumber\\
       &\hspace{2em}+ \sum_{i,j} w_{ij}(x_{j}-x_{i})\omega \omega^{\top} (x_{j}-x_{i})^{\top} \nonumber\\
       &\hspace{2em}- 2 \sum_{i,j}w_{ij}\boldsymbol{D}_{ij}\lVert x_j - x_i \rVert_2 \nonumber\\
       &\hspace{2em}- 2 \sum_{i,j} w_{ij}\boldsymbol{D}_{ij}\omega^{\top} (x_j - x_i)^{\top} + \sum_{i,j} w_{ij}\boldsymbol{D}_{ij}^{2}. 
       \label{eq: stress W sym sum ij}
\end{align}

The terms $\sum_{i,j} w_{ij} \lVert x_{j} - x_{i} \rVert_2^{2}$ and $\sum_{i,j}w_{ij}\boldsymbol{D_{ij}}\lVert x_j - x_i \rVert_2$ are the ones we would obtain in the traditional SMACOF algorithm \cite{groenen2016multidimensional}, and can be written, respectively, $\tr(\boldsymbol{X}^{\top}V\boldsymbol{X})$ and $\tr(\boldsymbol{X}^{\top}B(\boldsymbol{X})\boldsymbol{X})$, with $V$ and $B$ given by \cref{eq:V} and \cref{eq:B}.

The terms $\sum_{i,j} w_{ij}(x_{j}-x_{i})\omega \omega^{\top} (x_{j}-x_{i})^{\top}$ and $\sum_{i,j} w_{ij}\boldsymbol{D}_{ij}\omega^{\top} (x_j - x_i)^{\top}$ are specific to the Randers metric, and can be simply written as $\tr(\boldsymbol{X}^{\top}V\boldsymbol{X}\omega \omega^{\top})$ and 
$\tr((W^{\top}\odot \boldsymbol{D}^{\top} -W \odot \boldsymbol{D})\mathbbm{1}_{m}\omega^{\top} \boldsymbol{X}^{\top})$.

\subsection{Proof of \cref{prop:FSMACOF}}\label{sec:proof of fsmacof}

Our proof is based on the \emph{majorisation} approach \cite{leeuw1977application,groenen2016multidimensional}. Inspired by the traditional SMACOF algorithm, we aim to find a function $g(\cdot,\cdot)$ that satisfies all the following conditions for any points $\boldsymbol{X}$ and $\boldsymbol{Y}$:
\begin{enumerate}[label=(\roman*)]
    \item $\sigma^{2}(\boldsymbol{X}) = g(\boldsymbol{X},\boldsymbol{X})$,
    \item $\sigma^{2}(\boldsymbol{X}) \leq g(\boldsymbol{X},\boldsymbol{Y})$ for any $\boldsymbol{Y}$,
    \item $g(\boldsymbol{X},\boldsymbol{Y})$ 
    can be easily 
    minimised
    with respect to $\boldsymbol{X}$ for any $\boldsymbol{Y}$.
\end{enumerate}
For such a function $g$, the algorithm 
\begin{equation}
    \label{eq: majorisation update rule}
    \boldsymbol{X}^{(k+1)} = 
    \argmin
    \limits_{\boldsymbol{X}} g(\boldsymbol{X}, \boldsymbol{X}^{(k)})
\end{equation}
decreases the stress at each iteration as
\begin{equation}
    g(\boldsymbol{X}^{(k)},\boldsymbol{X}^{(k)}) \ge g(\boldsymbol{X}^{(k+1)}, \boldsymbol{X}^{(k)}) \ge g(\boldsymbol{X}^{(k+1)}, \boldsymbol{X}^{(k+1)}).
\end{equation}
 Since the stress $\sigma^2(\boldsymbol{X}^{(k)}) = g(\boldsymbol{X}^{(k)}, \boldsymbol{X}^{(k)})$ decreases at each iteration, the algorithm converges to a local minimum (sandwich theorem).

We now look for a suitable function $g$. From the derivation of the stress function in \cref{eq: stress W sym sum ij}, we have
\begin{align}
    \sigma^{2}(\boldsymbol{X}) &= \tr(\boldsymbol{X}^{\top}V\boldsymbol{X}) + \tr(\boldsymbol{X}^{\top}V\boldsymbol{X}\omega \omega^{\top})  \nonumber\\ 
    &\hspace{2em}+ 2 \tr(C\boldsymbol{X}^{\top}) - 2 \tr(\boldsymbol{X}^{\top}B(\boldsymbol{X})\boldsymbol{X}) ,
\end{align}
As in the traditional SMACOF \cite{borg2007modern}, the Cauchy-Schwarz inequality implies that
\begin{align}
    \tr(\boldsymbol{X}^{\top}B(\boldsymbol{X})\boldsymbol{X}) \geq \tr(\boldsymbol{X}^{\top}B(\boldsymbol{Y})\boldsymbol{Y}) = \mu(\boldsymbol{X},\boldsymbol{Y}).
\end{align}
We thus have
\begin{equation}
    \sigma^{2}(\boldsymbol{X}) \leq g(\boldsymbol{X},\boldsymbol{Y}) ,
\end{equation}
where
\begin{align}\label{eq:g}
    g(\boldsymbol{X},\boldsymbol{Y}) &= \tr(\boldsymbol{X}^{\top}V\boldsymbol{X}) + \tr(\boldsymbol{X}^{\top}V\boldsymbol{X}\omega \omega^{\top})  \nonumber\\ 
    &\hspace{2em}+ 2 \tr(C \boldsymbol{X}^{\top}) - 2 \mu(\boldsymbol{X},\boldsymbol{Y}) .
\end{align}

To implement the majorisation update rule (\cref{eq: majorisation update rule}), we need to compute the gradient
of $g(\cdot,\boldsymbol{X}^{(k)})$ (\cref{eq:g}).
We have
\begin{align}
\nabla_{\boldsymbol{X}}g(\boldsymbol{X}^{(k+1)},\boldsymbol{X}^{(k)}) &= 2 V \boldsymbol{X}^{(k+1)} +2 V \boldsymbol{X}^{(k+1)}\omega \omega^{\top} \nonumber\\
&\hspace{2em}+ 2C - 2 B(\boldsymbol{X}^{(k)})\boldsymbol{X}^{(k)} .
\end{align}
Since $\boldsymbol{X}^{(k+1)}$ 
minimises
$g$ (\cref{eq: majorisation update rule}), the first-order optimality conditions
lead to
\begin{equation}
2 V \boldsymbol{X}^{(k+1)} + 2 V \boldsymbol{X}^{(k+1)}\omega \omega^{\top} + 2C = 2 B(\boldsymbol{X}^{(k)})\boldsymbol{X}^{(k)}.
\label{eq: grad = 0 k+1 k before pseudo inverse}
\end{equation}

The terms in $\boldsymbol{X}^{(k+1)}$ are linear, we can thus pseudo-invert this system of equations to get the update rule. Recall how to rewrite a linear system of equations to only have the unknowns to the right of the coefficient matrix.

\begin{lemma}
    For any matrices $A$, $X$, and $C$, we have
    $$ A X B = C \iff (B^{\top} \otimes A)\vect(X) = \vect(C).$$
\end{lemma}

Applying this rewrite to the linear system of equations \cref{eq: grad = 0 k+1 k before pseudo inverse}, we get the desired update rule
\begin{equation}
    \vect(\boldsymbol{X}^{(k+1)}) = \\ K^{\dagger} \vect(B(\boldsymbol{X}^{(k)})\boldsymbol{X}^{(k)} - C) ,
\end{equation}
where $K =  (I_m + \omega\omega^\top) \otimes V$ is a Kronecker matrix.

\section{The Relationship between Current Fields and Randers Metrics}
\label{sec: current vs randers}

The search of shortest-time trajectories in a medium with time-independent wind is an old problem first studied by Ernst Zermelo \cite{zermelo1931navigationsproblem} and is called the \emph{Zermelo navigation problem}. In fact, it has turned out to be such an important question that it can be used to explain causality in space-time \cite{caponio2014wind}.
In the presence of wind $v(x)$, unit balls of the Finsler metric $F_x$ are offset by $v(x)$. To remain in a Finsler space, where $0$ is inside unit balls, the wind must have a small magnitude $F_x(-v(x)) < 1$. Note that in the presence of large winds, the wind implies irreversible displacements, explaining the irreversibility of time and causality in the world. However, the obtained metric in large winds is no longer a Finsler metric.

Consider the traditional case of a Riemannian manifold $\mathcal{X}$. For notational simplicity, we will drop the explicit dependence on $x$. The Riemannian metric is written as $R(u) = \lVert u\rVert_M$. Consider a wind with small magnitude $\lVert v\rVert_{M^{-1}} < 1$. The Zermelo metric $F$, which provides the Finsler metric measuring the traversal time of agents along curves on $\mathcal{X}$ with wind $v$ is given by the equation \cite{shen2003finsler}
\begin{equation}
    R\left(\frac{u}{F(u)} - v\right) = 1.
\end{equation}
Solving this equation with respect to $F(u)$ yields the Zermelo metric given by
\begin{equation}
    F(u) = \lVert u\rVert_{M_v} + \omega_v^\top u,
\end{equation}
where
\begin{align}
    M_v &= \frac{1}{(1-\lVert u \rVert_{M}^2)^2}
    \Big(Mvv^\top M + (1-\lVert v\rVert_{M}^2)M\Big), \\
    \omega_v &= - \frac{1}{1-\lVert v\rVert_{M}^2} M v.
\end{align}
The Zermelo metric is thus a Randers Finsler metric.
In particular, note that for the traditional isotropic Riemannian metric with $M = I$, and a small current $\lVert v\rVert_2^2 \ll 1$, then $M_v\approx M$ and the Randers drift component becomes $\omega_v \approx -v$.
As we work on synthetic current data with $M = I$, we make the simplifying approximation when computing the Zermelo-Randers metric that it is given by $F(u) = \lVert u\rVert_2 - v^\top u$. Thus our Randers linear drift component is given by the opposite of the current field.

\section{Wormhole Finsler MDS}
\label{sec: wormhole Finsler MDS}

Our Finsler MDS formulation allows to use non-uniform weights $w_{i,j}$ in the Finsler stress function, similar to regular MDS approaches. Here, we focus on generalising the recent state-of-the-art method WHCIE \cite{bracha2024wormhole} 
for computing theoretically guaranteed consistent pairs of points on manifolds sampled with missing parts. It was originally motivated for improving unsupervised shape matching to handle partial shapes by filtering out inconsistent pairs from the Gromov-Wassertein loss \cite{bracha2024unsupervised}. 
We first present the existing approaches in Riemannian manifolds and then focus on our generalisation to Finsler manifolds.

\hfill \break
\noindent \textbf{Riemannian wormhole criterion.}
Let $\tilde{\mathcal{X}}$ be a Riemannian data manifold (without missing parts) and $\tilde{\mathcal{Y}}\subset\tilde{\mathcal{X}}$ be a version of the data manifold that is missing some parts $\tilde{\mathcal{Y}}\neq \tilde{\mathcal{X}}$.
Let $\tilde{\boldsymbol{X}}$ be sampled data on $\tilde{\mathcal{Y}}$ (and thus also on $\tilde{\mathcal{X}}$).
Data dissimilarities are computed as shortest path distances. 
However, depending on whether we are given the full manifold $\tilde{\mathcal{X}}$ or the partial one $\tilde{\mathcal{Y}}$, the computed data dissimilarities $D_{\tilde{\mathcal{X}}}$ and $D_{\tilde{\mathcal{Y}}}$, computed respectively on $\tilde{\mathcal{X}}$ and $\tilde{\mathcal{Y}}$, might differ. This is due to the fact that geodesic trajectories in the full manifold $\tilde{\mathcal{X}}$ might pass through missing parts of $\tilde{\mathcal{Y}}$ making shortest paths on $\tilde{\mathcal{Y}}$ longer for some pairs of points.
As the data dissimilarities differ, optimising the stress function with each of them and using the same uniform weight scheme $w_{i,j} = 1$ for all pairs will lead to different embeddings $\boldsymbol{X}\neq\boldsymbol{Y}$.
The objective here is to design a different strategy on the weights $w_{i,j}$ such that the resulting embeddings are as close as possible $\boldsymbol{X}\approx \boldsymbol{Y}$, meaning that the scheme is robust to missing parts.

Pairs of points $x_i$ and $x_j$ are said to be consistent if their shortest path distances $(D_{\tilde{\mathcal{X}}})_{i,j} = (D_{\tilde{\mathcal{Y}}})_{i,j}$ are consistent on the full $\tilde{\mathcal{X}}$ and partial $\tilde{\mathcal{Y}}$ shapes. In practice, a majority of pairs is consistent \cite{bracha2024wormhole}, but a significant amount of pairs are inconsistent, leading to incorrect geodesic dissimilarity estimates in the partial case affecting the embedding. To mitigate this effect, a natural approach is to filter out inconsistent pairs my masking out their contribution to the stress function. This translates to choosing a weight scheme $w_{i,j} \in\{0,1\}$, with $w_{i,j} = 1$ only for consistent pairs. One way of proceeding is to use heuristics for short distance computations and focus only on local pairs \cite{schwartz2019intrinsic}. More recently, another paradigm has shown impressive results \cite{rosman2010nonlinear,bracha2024wormhole}. Rather than focusing on local pairs, the idea is to design a criterion that can guarantee whether a pair is consistent. Guaranteeing means that there is theoretically no false alarm possible by the criterion: only consistent pairs are found. More general criteria find more consistent pairs, allowing the method to use more non-perturbed information to find the embedding.

A common misconception is to believe that shortest paths not intersecting the boundary $\tilde{\mathcal{B}} = \delta \tilde{\mathcal{Y}}$ provide consistent pairs, as was debunked in \cite{bracha2024wormhole}. Rather than focusing on the intersection with the boundary of the partial manifold, the distances to the boundary were used to define the criteria.
Let $\tilde{x}_{i_{\tilde{\mathcal{B}}}}$ and $\tilde{x}_{j_{\tilde{\mathcal{B}}}}$ be the closest boundary points to $\tilde{x}_i$ and $\tilde{x}_j$ on the partial shape $\tilde{\mathcal{Y}}$,
\begin{equation}
    \tilde{x}_{i_{\tilde{\mathcal{B}}}} = \argmin\limits_{\tilde{x}_b\in \tilde{\mathcal{B}}} (\boldsymbol{D}_{\tilde{\mathcal{Y}}})_{i,b}
    \quad \text{and} \quad
    \tilde{x}_{j_{\tilde{\mathcal{B}}}} = \argmin\limits_{\tilde{x}_b\in \tilde{\mathcal{B}}} (\boldsymbol{D}_{\tilde{\mathcal{Y}}})_{j,b}.
\end{equation}
In \cite{bronstein2006efficient,rosman2010nonlinear}, 
a first criterion $\mathcal{C}_\mathcal{T}:\tilde{\mathcal{Y}}\times\tilde{\mathcal{Y}}\to \{0,1\}$ was proposed
\begin{equation}
    \mathcal{C}_\mathcal{T}(\tilde{x}_i, \tilde{x}_j) = 
    \mathbbm{1}_{
    (\boldsymbol{D}_{\tilde{\mathcal{Y}}})_{i,j} \le (\boldsymbol{D}_{\tilde{\mathcal{Y}}})_{i, i_{\tilde{\mathcal{B}}}} + (\boldsymbol{D}_{\tilde{\mathcal{Y}}})_{j, j_{\tilde{\mathcal{B}}}}
    },
\end{equation}
where $\mathbbm{1}$ is the indicator function. The idea behind this criterion is that if geodesic paths on the full manifold $\tilde{\mathcal{X}}$ between points $\tilde{x}_i$ and $\tilde{x}_j$ should pass through missing parts in $\tilde{\mathcal{Y}}$, then their length is at least the sum of the distances to the boundary $\tilde{\mathcal{B}}$.
However, this intrinsic criterion is particularly conservative as it discards the length of this trajectory between boundary points, since information on the manifold is lost inside missing parts. Recently, \cite{bracha2024wormhole} lifted extrinsic information to provide a worst case bound on the length of paths between boundary points. If the Riemannian metric on the manifold is the standard one given by the identity matrix, then trajectories between boundary points on the manifold are at least longer than the length of the straight segment in the original Euclidean embedding space $\mathbb{R}^n$
\begin{equation}
    \label{eq: Euclidean straight line shortest}
    (\boldsymbol{D}_{\tilde{\mathcal{X}}})_{b_1, b_2} \ge d_E(\tilde{x}_{b_1}, \tilde{x}_{b_2})
\end{equation}
for any boundary points $\tilde{x}_{b_1}$ and $\tilde{x}_{b_2}$. From this simple observation, \cite{bracha2024wormhole} generalised the $\mathcal{C}_{\mathcal{T}}$ criterion to the \emph{wormhole criterion} $\mathcal{C}_{\mathcal{W}}:\tilde{\mathcal{Y}}\times\tilde{\mathcal{Y}}\to\{0,1\}$ defined as
\begin{equation}
    \mathcal{C}_{\mathcal{W}}(\tilde{x}_i,\tilde{x}_j) = \mathbbm{1}_{(\boldsymbol{D}_{\tilde{\mathcal{Y}}})_{i,j} \le \boldsymbol{K}_{i,j}^E
    },
\end{equation}
where the threshold matrix $\boldsymbol{K}_{i,j}^E$ is computed as
\begin{equation}
    \boldsymbol{K}_{i,j}^E = \min\limits_{\tilde{x}_{b_1},\tilde{x}_{b_2}\in B} 
    (\boldsymbol{D}_{\tilde{\mathcal{Y}}})_{i, b_1} + (\boldsymbol{D}_{\tilde{\mathcal{Y}}})_{j, b_2} + d_E(\tilde{x}_{b_1}, \tilde{x}_{b_2}).
\end{equation}
For more general Riemannian metrics on the manifold, \cite{bracha2024wormhole} showed how to generalise the wormhole criterion. The idea is to provide a worst case bound on the distance of each infinitesimally small Euclidean arclength step along the straight Euclidean segment between boundary points. 
Denote $\lambda_{\tilde{M}} >0$ to be the minimum eigenvalue of the Riemannian metric $\tilde{M}$ over the full manifold $\tilde{\mathcal{X}}$, and can be assumed to be given. By bounding the Riemannian length of Euclidean arclength steps along curves, \cref{eq: Euclidean straight line shortest} becomes
\begin{equation}
    \label{eq: Riemann and Euclidean straight line shortest}
    (\boldsymbol{D}_{\tilde{\mathcal{X}}})_{b_1,b_2} \ge \sqrt{\lambda_{\tilde{M}}}d_E(\tilde{x}_{b_1}, \tilde{x}_{b_2})
\end{equation}
for any boundary points $\tilde{x}_{b_1}, \tilde{x}_{b_2}\in \tilde{\mathcal{B}}$. The wormhole criterion then becomes
\begin{equation}
    \mathcal{C}_{\mathcal{W}}(\tilde{x}_i,\tilde{x}_j) = \mathbbm{1}_{(\boldsymbol{D}_{\tilde{\mathcal{Y}}})_{i,j} \le \boldsymbol{K}_{i,j}^R
    }
\end{equation}
where the generalised Riemannian threshold matrix $\boldsymbol{K}^R$ is now
\begin{equation}
    \boldsymbol{K}_{i,j}^R = \min\limits_{\tilde{x}_{b_1},\tilde{x}_{b_2}\in B} 
    (\boldsymbol{D}_{\tilde{\mathcal{Y}}})_{i, b_1} + (\boldsymbol{D}_{\tilde{\mathcal{Y}}})_{j, b_2} + \sqrt{\lambda_{\tilde{M}}}d_E(\tilde{x}_{b_1}, \tilde{x}_{b_2}).
\end{equation}
The criteria $\mathcal{C}_\mathcal{T}(x_i, x_j)$ and $\mathcal{C}_\mathcal{W}(x_i, x_j)$ are chosen to be the weights $w_{i,j}$ for the TCIE \cite{rosman2010nonlinear} and WHCIE \cite{bracha2024wormhole} methods respectively. In particular, WHCIE demonstrates impressive robustness
and forms the current state-of-the-art in finding consistent pairs on Riemannian manifolds.

The core idea behind the wormhole criterion is in \cref{eq: Euclidean straight line shortest,eq: Riemann and Euclidean straight line shortest}, that find how to lower bound the manifold's metric length of Euclidean arclength infinitesimal steps. We propose to take this idea and apply it to Finsler manifolds.

\hfill \break
\noindent \textbf{Finsler wormhole criterion.}
Assume now that the data manifold $\tilde{\mathcal{X}}$ is equipped with a Finsler metric $\tilde{F}$ and that there exists $C_{\tilde{F}} > 0$ such that the Finsler length of infinitesimal Euclidean arclength steps $d\tilde{s}$ is bounded by $\tilde{F}_{\tilde{x}}(d\tilde{s})  \ge C_{\tilde{F}} \lVert d\tilde{s} \rVert_2$. Then the Finsler length of curves between boundary points can be lower bounded using the Euclidean embedding distance.

\begin{proposition}
    \label{prop: Finsler and Euclidean straight line shortest}
    The Finsler distance on the Finsler manifold $\tilde{\mathcal{X}}$ between any points $x_i$ and $x_j$ is lower bounded by
    $$ (\boldsymbol{D}_{\tilde{\mathcal{X}}})_{i, j} \ge C_{\tilde{F}} d_E(\tilde{x}_i, \tilde{x}_j)
    $$
\end{proposition}

\begin{proof}
    The proof is an immediate generalisation of the arguments in the Riemannian case. By integrating the lower bound on Euclidean arclength steps $\tilde{F}_{\tilde{x}}(d\tilde{s})  \ge C_{\tilde{F}} \lVert d\tilde{s} \rVert_2$, and since the euclidean length of any curve between $x_i$ and $x_j$ is at least that of the Euclidean straight segment between them, we get the desired lower bound.
\end{proof}

Denote $\boldsymbol{K}^F$ the generalised Finsler threshold matrix
\begin{equation}
    \boldsymbol{K}_{i,j}^F =  \min\limits_{\tilde{x}_{b_1},\tilde{x}_{b_2}\in B} 
    (\boldsymbol{D}_{\tilde{\mathcal{Y}}})_{i, b_1} + (\boldsymbol{D}_{\tilde{\mathcal{Y}}})_{b_2, j} + C_{\tilde{F}}d_E(\tilde{x}_{b_1}, \tilde{x}_{b_2}).
\end{equation}
We can then define the Finsler wormhole criterion $\mathcal{C}_{{\mathcal{W}_{\mathcal{F}}}}$.

\begin{definition}[Finsler wormhole criterion]
    The Finsler wormhole criterion $\mathcal{C}_{{\mathcal{W}_{\mathcal{F}}}}$ is defined as
    $$\mathcal{C}_{{\mathcal{W}_{\mathcal{F}}}}(\tilde{x}_i,\tilde{x}_j) = \mathbbm{1}_{(\boldsymbol{D}_{\tilde{\mathcal{Y}}})_{i,j} \le \boldsymbol{K}_{i,j}^F
    }.
    $$
\end{definition}

By construction, the Finsler wormhole criterion only finds consistent pairs.

\begin{theorem}[$\mathcal{C}_{{\mathcal{W}_{\mathcal{F}}}}$ guarantees consistent pairs]
    The Finsler wormhole criterion guarantees found pairs to be consistent.
\end{theorem}

\begin{proof}
    The proof follows the exact same arguments as in the Riemannian case, where now \cref{eq: Euclidean straight line shortest,eq: Riemann and Euclidean straight line shortest} are replaced with \cref{prop: Finsler and Euclidean straight line shortest}.
\end{proof}

We thus propose the weight scheme $w_{i,j} = \mathcal{C}_{{\mathcal{W}_{\mathcal{F}}}}(\tilde{x}_i,\tilde{x}_j)$ for Finsler MDS to provide robust embeddings to missing components. For optimisation algorithms requiring a symmetric weight scheme, such as our Finsler Smacof algorithm, we symmetrise it by taking the intersection $w_{i,j} = \sqrt{\mathcal{C}_{{\mathcal{W}_{\mathcal{F}}}}(\tilde{x}_i,\tilde{x}_j) \mathcal{C}_{{\mathcal{W}_{\mathcal{F}}}}(\tilde{x}_j,\tilde{x}_i)}$. Note that the square root is superfluous for binary criteria, but is not so when considering soft masks. In \cite{bracha2024wormhole}, the criterion is sometimes softened by considering the ratio between the computed shortest path lengths and the criterion matrix, and cutting it off to 1. This allows to take into account almost consistent pairs where there is only a small perturbation of the true geodesic distance, providing a reasonable compromise between accuracy and amount of data to rely on. We can soften our criterion in the same fashion by taking: $\min\Big\{\tfrac{\boldsymbol{K}^F}{\boldsymbol{D}_{\tilde{\mathcal{Y}}}}, 1 \Big\}$.

We now show in a useful example how to derive the Finsler constant $C_{\tilde{F}}$ when the Finsler metric is a Randers metric  with isotropic uniform Riemannian component $\tilde{F}_{\tilde{x}}(u) = \lVert u\rVert_2 + \tilde{\omega}(\tilde{x})^\top u$. Taking $u=d\tilde{s}$ to be an infinitesimal Euclidean arclength tangent vector, its Finsler length becomes minimal when $d\tilde{s}$ is oppositely aligned with the Randers drift component $\tilde{\omega}(\tilde{x})$. This leads to $\tilde{F}_{\tilde{x}}(d\tilde{s}) \ge (1-\lVert\tilde{\omega}(\tilde{x})\rVert_2) \lVert d\tilde{s}\rVert_2$. Assuming the knowledge of $\tilde{\alpha}_{\max} = \max\limits_{\tilde{x}} \lVert\tilde{\omega}(\tilde{x})\rVert_2 < 1$, for instance if we are provided with the maximum possible norm of the current on the manifold, we get 
$\tilde{F}_{\tilde{x}}(d\tilde{s}) \ge (1-\tilde{\alpha}_{\max}) \lVert d\tilde{s}\rVert_2$, meaning that $C_{\tilde{F}} = 1-\tilde{\alpha}_{\max}$.

\section{Implementation Details and Additional Experiments}
\label{sec: additional details and res experiments}

\subsection{Data Visualisation Experiments}\label{subsec:data_visualisation_experiments}

We describe the implementation details in \cref{sec: visualisation exp implementation considerations} 
of experiments in \cref{subsec:data_visualisation}
and present additional visualisation results in \cref{subsec:data_visualisation}.
These simple experiments do not require any advanced hardware, e.g.\ a commercial CPU suffices.

\subsubsection{Implementation Considerations}
\label{sec: visualisation exp implementation considerations}

In the visualisation experiments, we embed data with Finsler MDS into the canonical Randers space $\mathcal{X} = \mathbb{R}^m$, with $m\in\{2,3\}$. The canonical Randers metric is chosen to have the fixed asymmetry level $\alpha = 0.5$. 
All Finsler MDS embeddings for visualisation are computed with the Finsler SMACOF algorithm. Unless specified otherwise, they use uniform weights $w_{i,j}$.
Recall that the traditional SMACOF algorithm is well-known to be sensitive to initialisation. To avoid getting stuck in bad local minima, it is considered standard practice to initialise it with the Isomap \cite{schwartz1989numerical,tenenbaum2000global} embedding, even if the weights $w_{i,j}$ are not uniform.
Following this idea, we initialise the SMACOF algorithm with the Isomap embedding to $\mathbb{R}^m$ applied to the symmetrised dissimilarity matrix $\boldsymbol{D}^S = \tfrac{\boldsymbol{D} + \boldsymbol{D}^\top}{2}$.

In practice, we found that pseudo-inverting the $K$ matrix for the Finsler SMACOF update (see \cref{prop:FSMACOF}) was slow and unstable when there are many data points. To overcome this issue, we first multiplied \cref{eq: grad = 0 k+1 k before pseudo inverse} by $V^\top$ leading to a more stable update rule requiring the pseudo-inversion of a symmetric matrix
\begin{equation}
    \vect(\boldsymbol{X}^{k+1}) = (K')^\dagger\vect( B'(\boldsymbol{X}^k)\boldsymbol{X}^k - C'),
\end{equation}
where the matrices $K'$, $B'(\boldsymbol{X}^k)$, and $C'$ are the modified matrices $K' = (I_m + \omega\omega^\top) \otimes (V^\top V)$, $B'(\boldsymbol{X}^k) = V^\top B(\boldsymbol{X}^k)$, and $C' = V^\top C$. In addition, we resorted to the Generalized Minimal Residual method (GMRES) \cite{saad1986gmres}, which is a fast alternative solver of linear systems, bypassing the need to compute the Moore-Penrose pseudo-inverse of the large matrix $(K')^\dagger$ when the number of points $N$ is large.
We share the seeded code to reproduce our data and results.

\hfill \break
\noindent \textbf{Asymmetric Manifold Flattening.}
In this experiment from the main paper and the additional one in the supplementary, we sample $N = 3000$ i.i.d.\ random vertices from the Swiss roll. The unit Euclidean vector $\tilde{\omega}$, giving the direction of the Randers metric equipping the Swiss roll, is chosen to be intrinsically uniform along the length of the Swiss roll. Note that although they are intrinsically uniform in the tangent planes $\mathcal{T}_{\tilde{x}}\tilde{\mathcal{X}} = \mathbb{R}^2$, they are not uniform extrinsically when rotating these planes to be tangent to the original embedding of the Swiss roll, as shown for instance in \cref{fig: swiss role single alpha no isomap}. Denote $\hat{\omega}(\tilde{x})\in\mathbb{R}^3$ the extrinsic embedding of $\tilde{\omega}$ in the original embedding space $\mathbb{R}^3$ of the Swiss roll manifold $\tilde{\mathcal{X}}$.
To compute the asymmetric geodesic distances, we compute the symmetric k-Nearest Neighbour (kNN) graph, with $k=10$, based on the Euclidean distances in $\mathbb{R}^3$. Once the logical graph is computed, we compute the distances on these edges using a first order approximation. If points $\tilde{x}_i$ and $\tilde{x}_j$ are neighbours, we approximate $d_{\tilde{F}^{\tilde{\alpha}}} (x_i,x_j)\approx \lVert x_j - x_i \rVert_2 + \tilde{\alpha}\hat{\omega}(x_i)^\top (x_j-x_i)$, and assign this distance to the directed edge from node $i$ to node $j$, and vice versa for the directed edge from node $j$ to node $i$. This procedure, which generalises the standard Isomap \cite{schwartz1989numerical,tenenbaum2000global} approach, constructs an asymmetric weighted kNN directed graph. We can now apply Dijkstra's algorithm \cite{dijkstra1959note} to compute the approximate geodesic distances between all pairs of points. The results form the dissimilarity matrix $\boldsymbol{D}$, which is the input for the embedding algorithm. The result in \cref{fig: swiss role single alpha no isomap} corresponds to $\tilde{\alpha} = 0.3$.

\hfill \break
\noindent \textbf{Robustness to Holes.}
In this experiment, $2000$ i.i.d. points are sampled on the full Swiss roll, but points falling within a rectangular region encoding the hole are removed. The Randers metric equipping the manifold $\tilde{\mathcal{X}}$ is the same as in the \emph{Asymmetric Manifold Flatenning} experiments on the Swiss roll with $\tilde{\alpha} = 0.5$. We apply the same algorithm to compute the Randers distance between points, with $k=15$ in the kNN graph construction. To create an embedding that is robust to the missing part, the weights are given by the binary Finsler wormhole criterion, logically symmetrised: $w_{i,j} = \sqrt{\mathcal{C}_{{\mathcal{W}_{\mathcal{F}}}}(\tilde{x}_i,\tilde{x}_j) \mathcal{C}_{{\mathcal{W}_{\mathcal{F}}}}(\tilde{x}_j,\tilde{x}_i)}$ (see \cref{sec: wormhole Finsler MDS}). To compute the wormhole criterion, we assume that the metric is behaved in the missing parts similarly to the rest of the data, leading to a choice of $\tilde{\alpha}_{\max} = \tilde{\alpha}$ to compute the constant $C_{\tilde{F}}$ in \cref{prop: Finsler and Euclidean straight line shortest}.

\hfill \break
\noindent \textbf{Unflattening Current Maps.}
In this experiment from the main paper and the additional one in the supplementary (corresponding to \cref{fig: current map river}), we sample $N$ i.i.d.\ random points in a rectangular region $\tilde{\Omega}$ of the plane $\mathbb{R}^2$.
Given an unconstrained current $\breve{v}(\tilde{x}_i)\in\mathbb{R}^2$ at any point $\tilde{x}_i \in\tilde{\Omega}$, the current is then chosen to be $\tilde{v}(\tilde{x}_i) = \tilde{\alpha} \tfrac{\breve{v}(\tilde{x}_i)}{\max_{j}\lVert \breve{v}(\tilde{x}_j)\rVert_2}$. 
The Randers metric at the sampled point $\tilde{x}_i$ is then chosen to be $\tilde{F}(\tilde{x}_i) = \lVert u\rVert_2 - \tilde{v}(\tilde{x}_i)^\top u$ (see \cref{sec: current vs randers}). We then apply the same algorithm as in the \emph{Asymmetric Manifold Flattening} experiment on the Swiss roll to compute Randers geodesic distances, using a kNN graph with $k=10$ neighbours. Note that since the original space and its tangent space coincide $\tilde{\mathcal{X}} = \mathcal{T}_{\tilde{x}}\tilde{\mathcal{X}}$ at all points $\tilde{x}$, the extrinsic embedding of the drift component $\tilde{\omega}(\tilde{x}) = -\tilde{v}(\tilde{x})$ is the same as its intrinsic version $\hat{\omega}(\tilde{x}) = \tilde{\omega}(\tilde{x})$.

For the experiment in the main paper, we sample $N=2000$ points from the domain $\tilde{\Omega} = [0,10]^2$. 
At any point $\tilde{x}_i = (\tilde{x}_i^{(1)},\tilde{x}_i^{(2)})^\top \in \Omega$, we define the unconstrained current field $\breve{v}(\tilde{x}_i) = \big(\sin(\nu \tilde{x}_i^{(1)}) + \cos(\nu \tilde{x}_i^{(2)}),
\cos(\nu \tilde{x}_i^{(1)}) - \sin(\nu \tilde{x}_i^{(2)})
\big)^\top$, with $\nu=2$. The current field is constructed with $\tilde{\alpha} = 0.5$.
For the river experiment in the supplementary, we sample $N=1000$ points from the domain $\tilde{\Omega} = [0,10]\times [0,1]$. The unconstrained current at point $\tilde{x}_i$ is given by $\breve{v}(\tilde{x}_i) = (1-|2\tilde{x}_i^{(2)}-1|,0)^\top$. The current is then constructed with $\tilde{\alpha} = 0.2$.

\hfill \break
\noindent \textbf{Revealing Graph Hierarchies}
In this experiment, we construct a full and complete binary tree of depth $h = 7$, having thus $N = 2^{h + 1} - 1 = 255$ nodes. The edge from a parent to it child is given the weight of $0.5$, whereas the edge from a child to its parent has a weight of $1.5$. Additionally, we add undirected edges between all nodes at the same height, with a weight of $0.1$.
Given two nodes connected by an edge, their distances is given by the edge weight. Asymmetric geodesic distances between any two nodes are then computed using Dijkstra's algorithm, which constitute the dissimilarity matrix $\boldsymbol{D}$.

\subsubsection{Additional Results}
\label{sec: additional results visualisation exp}

\begin{figure}[ht]
  \centering
   \includegraphics[width=\columnwidth]{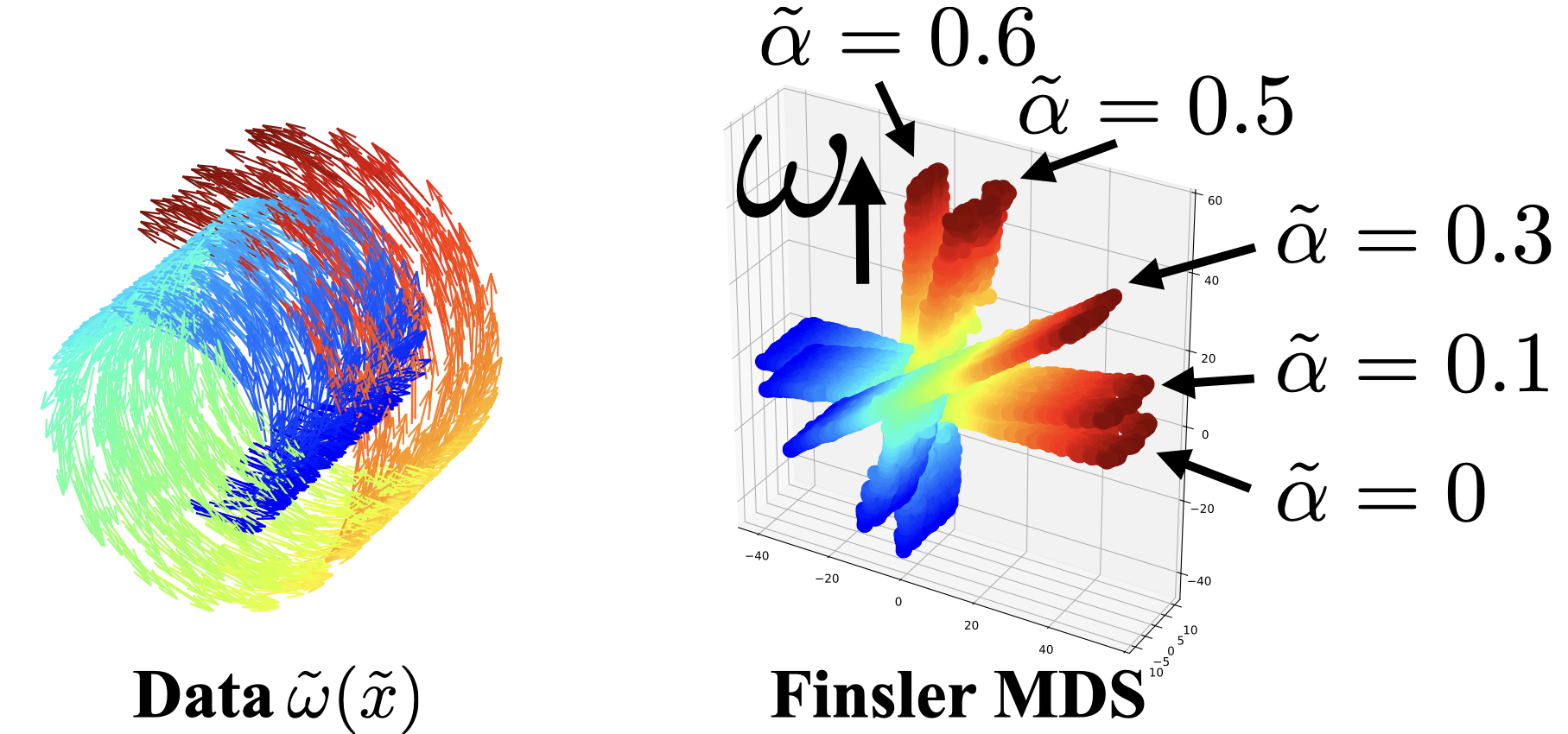}
   \caption{Flattening the Swiss roll equipped with a Randers metric $\tilde{F}^{\tilde{\alpha}}$ with various asymmetry levels $\tilde{\alpha}$ given by $\tilde{F}_{\tilde{x}}^{\tilde{\alpha}}(u) = \lVert u\rVert_2 + \tilde{\alpha}\tilde{\omega}^\top u$ and $\lVert \tilde{\omega} \rVert_2 = 1$.
   We superimpose on the right the resulting Finsler MDS embeddings in the 3D canonical Randers space with fixed asymmetry $\alpha = \lVert \omega \rVert_2$.
   }
   \label{fig: swiss role full summary}
\end{figure}

\hfill \break
\noindent \textbf{Asymmetric Manifold Flattening.}
By changing the value of $\tilde{\alpha}$, we vary the amount of asymmetry on the Swiss roll. However, in this experiment, we do not change the asymmetry measure of the canonical Randers space of the embeddings: $\alpha$ is fixed. We superimpose in \cref{fig: swiss role full summary} the resulting Finsler MDS embeddings for various asymmetry levels of the data $\tilde{\alpha}\in\{0,0.1,0.3,0.5,0.6\}$. In all cases, the embedded Swiss roll resembles a flat 2D band in 3D, albeit with varying vertical orientation. As expected, the higher the value of  $\tilde{\alpha}$, the more the embedded Swiss roll becomes vertical along the axis $\vec{z}$ of asymmetry. When the data is (close to) symmetric, i.e.\ $\tilde{\alpha}$ is (close to) $0$, the embedding is (close to) aligned with the $xy$ plane. 
Finsler MDS thus not only provides embeddings preserving the manifold structure, but its verticality also provides an intuitive visual cue encoding the asymmetry of the data.

\begin{figure}[ht]
  \centering
  \includegraphics[width=\columnwidth]{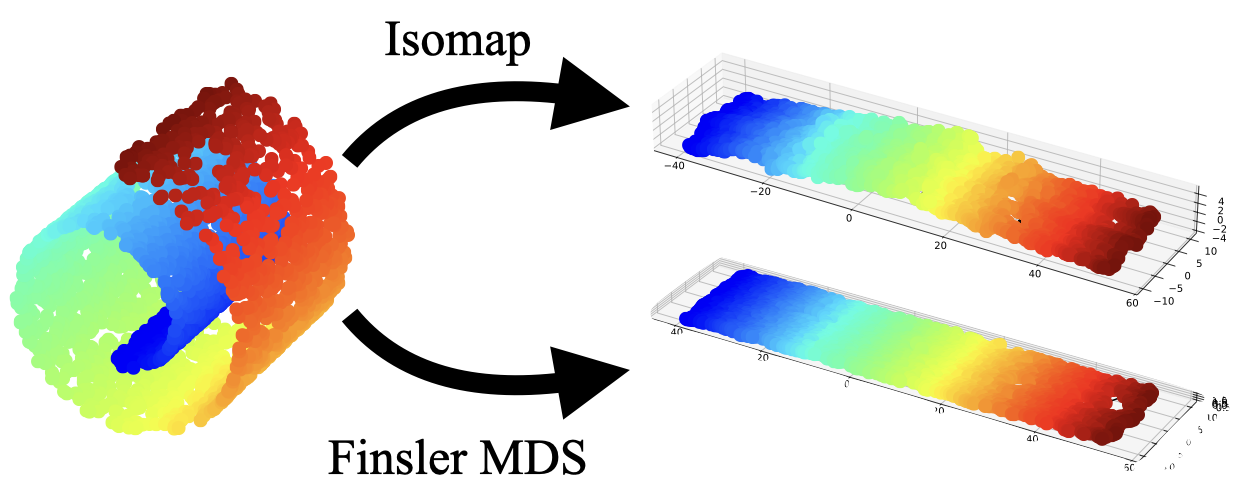}
   \caption{Flattening the original symmetric Swiss roll. The embedding is either into the Euclidean space $\mathbb{R}^3$ with Isomap or into the canonical Randers space $\mathbb{R}^3$ with our Finsler MDS. Finsler MDS provides robust embeddings that generalises traditional symmetric embedding methods on symmetric data while revealing the additional information that the data is symmetric.}
   \label{fig: swiss roll full symmetric}
\end{figure}

\hfill \break
\noindent \textbf{Symmetric Manifold Flattening.} We focus on the embedding of the vanilla symmetric Swiss roll, i.e.\ $\tilde{\alpha} = 0$, to $\mathbb{R}^3$ using either the traditional MDS, with Isomap, or Finsler MDS, with our Finsler SMACOF algorithm. 
These results
are presented without other values of $\tilde{\alpha}$ in \cref{fig: swiss roll full symmetric}.
For the Isomap embedding to $\mathbb{R}^3$, the Swiss roll is not perfectly flattened in the $xy$ hyperplane. This incorrectly suggests that the Swiss roll is not a flat Riemannian structure, i.e.\ with effectively $0$ Gaussian curvature. This error is due to small noise in the estimate of the distance matrix $\boldsymbol{D}$ as Dijkstra's algorithm only provides an approximation to geodesic distances as geodesic paths are constrained to live on the neighbourhood graph constructed from the data. To avoid this issue, the Swiss roll is usually embedded to $\mathbb{R}^2$, yielding the desired 2D flattened Swiss roll rectangle. In contrast, our Finsler MDS embedding to the canonical Randers space $\mathbb{R}^3$ is flattened to the $xy$ plane and is similar to the ideal 2D Isomap embedding, as predicted by \cref{th: canonical randers generalises riemann with extra dim}. Additionally, our embedding also provides the information that the original Swiss roll is a symmetric structure as all embedded points have the same height. As such, Finsler MDS not only robustly provides superior embeddings for symmetric data that generalise the traditional methods, it also yields additional information compared to them.

\begin{figure*}[t]
  \centering
   \includegraphics[width=\textwidth]{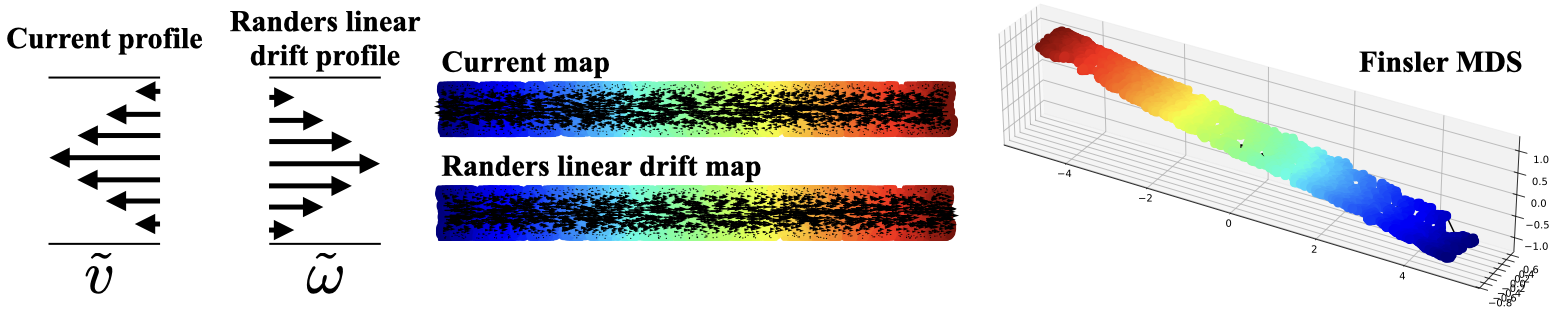}
   \caption{Embedding of the river map with a fixed current profile $\tilde{v}$. The associated Randers drift component $\tilde{\omega}$ is in the opposite direction. Plotting arrows on the map might lead to cluttered visualisations that make the asymmetry difficult to read even in this simple toy example. In contrast the Finsler MDS embedding clearly reveals the asymmetric nature of the river while preserving its spatial straight property.}
   \label{fig: current map river}
\end{figure*}

\hfill \break
\noindent \textbf{Unflattening Current Maps.} In addition to the unflattening of the current map in \cref{fig: current maps sea}, we also embed using Finsler MDS the classic river manifold in \cref{fig: current map river}. As explained in \cref{sec: current vs randers}, from a timewise perspective, we equip the river with a Randers field with $\tilde{\omega} = -\tilde{v}$, where $\tilde{v}$ is the current field. The Finsler MDS embedding of the river leads to an intuitive embedding clearly revealing the existence of asymmetry between points upstream and downstream. This contrasts with the original current map, even when enriched with arrows to artificially break the Euclidean symmetry, as they can be difficult to discern when numerous or with low magnitude.

\subsection{Digraph Embedding and Link Prediction Experiments Implementation Details}\label{subseec:digraph_embedding_and_link_prediction}

We describe the setups and additional details of our experiments in \cref{sec:digraph_embedding}. 
The experiments are performed on a NVIDIA DGX A100 GPU.

\begin{table}[t]
\Huge
\centering
\resizebox{\columnwidth}{!}{%
\begin{tabular}{lcccccccr}
\toprule
& Dataset & Cora & Citeseer   & Gr-QC& Chameleon & Squirrel &  Arxiv-Year\\
\midrule
& $\abs{\mathcal{V}}$  & 2,708 & 3,327 &  5,242 & 2,277 & 5,201  & 169,343 \\
& $\abs{\mathcal{E}}$ &  5,429 & 4,552   & 14,496 & 31,371 & 198,353 &  1,166,243 \\
\bottomrule
\end{tabular}
}
\caption{Summary of dataset statistics for link prediction tasks. We note $\abs{\mathcal{V}}$ and $\abs{\mathcal{E}}$ the numbers of nodes and edges, respectively.
}
\label{tab:dataset}
\end{table}

\hfill \break
\noindent \textbf{Datasets. }
For both the digraph embedding and link prediction tasks, we evaluate on six publicly available directed graph datasets: the citation networks Cora \cite{sen2008collective} and Citeseer \cite{yang2016revisiting}, the arXiv collaboration network in general relativity and quantum cosmology (Gr-QC) \cite{leskovec2007graph}, and three heterophilic graphs: Chameleon, Squirrel \cite{rozemberczki2021multi}, and Arxiv-Year \cite{hu2020open}. 
The detailed statistics of these benchmarks are
summarized in \cref{tab:dataset}.

\hfill \break
\noindent \textbf{Digraph Embedding Baseline. }
To utilize the directional property for learning efficient representations, we propose computing embeddings in Finsler space instead of Euclidean space. 
We note that while \cite{lopez2021symmetric} explores a Finsler-Riemannian framework for graph embedding, their approach is not applicable to directed graphs. Therefore, we do not include comparisons between their framework and our Finsler representation for digraph embedding.

\hfill \break
\noindent \textbf{Link Prediction Baseline. }
For link prediction tasks, we compare our method with NERD \cite{khosla2020node}, DiGCN \cite{tong2020digraph}, MagNet \cite{zhang2021magnet}, DiGAE \cite{kollias2022directed}, ODIN \cite{yoo2023disentangling}, and DUPLEX \cite{ke2024duplex}. 
NERD is a shallow method that uses node semantics based on a random walk strategy to sample node neighbourhoods from a directed graph.
DiGCN introduces a spectral Graph Neural Network (GNN) model built on digraph convolution, utilizing Personalized PageRank as its foundation.
MagNet proposes the magnetic Laplacian to define graph convolutions. 
Both DiGCN and MagNet are spectral-based methods. 
DiGAE is a digraph autoencoder model that employs a directed GCN as its encoder.
ODIN is a recent shallow method that learns multiple embeddings per node to model directed edge formation factors while disentangling interest factors from in-degree and out-degree biases.
DUPLEX employs dual graph attention network encoders that operate on a Hermitian adjacency matrix.

\hfill \break
\noindent \textbf{Digraph Embedding Setup. }
To assess the capacity of the Euclidean and Finsler representations, 
we embed the full data
with a Multi-Layer Perceptron (MLP)
and compute the pairwise distances in the embedding space.
We implement our proposed method using PyTorch \cite{paszke2019pytorch}.
We train the Euclidean embedding with 
Euclidean stochastic gradient descent.
Riemannian stochastic gradient descent \cite{bonnabel2013stochastic} generalises classical stochastic gradient descent to optimization on Riemannian manifolds by replacing Euclidean updates with retractions that map stochastic gradients from tangent spaces back onto the manifold.
For the Finsler embedding, we train the embedding by adapting the Riemannian stochastic gradient descent
from the Riemannian metric with the canonical Randers metric defined in \cref{sec:Finsler Multi-Dimensional Scaling}.
We consider the embedding space $\mathbb{R}^m$ of various dimensions $m\in\{2,5,10,50\}$.
For each dimension $m$, we use the Optuna \cite{akiba2019optuna} hyperparameter optimiser to choose
the learning rate, the number of hidden layers, the hidden dimension, and the dropout probability within candidate sets.
These candidate sets are
$\{5e^{-1}, 3e^{-1}, 2e^{-1}, 1e^{-1}, 5e^{-2}, 1e^{-2}, 5e^{-3}, 1e^{-3}\}$ for the learning rate,  
$\{1, 2, 3, \cdots,10\}$ for the number of 
hidden
layers,
$\{64,  128, 256, 512\}$ for the  hidden dimension,
and 
$\{0, 0.1, 0.2, \cdots, 0.9\}$
for the dropout probability.

\hfill \break
\noindent \textbf{Link Prediction Setup. }
We evaluate on two types of link prediction tasks. The first task involves predicting the direction of edges between vertex pairs $u$ and $v$, where
it is known that there exists an edge between the two vertices but not its direction:
$(u, v)\in \mathcal{E}$ or $(v, u)\in \mathcal{E}$.
The second task focuses on existence prediction, where the goal is to determine whether $(u, v)\in \mathcal{E}$, considering 
vertex pairs $(u, v)$. 
For link prediction tasks, we 
divide the graph datasets,
by partitioning the edges randomly while preserving the graph connectivity, 
into 80\% (of edges) for training, 15\% (of edges) for testing, and 5\% (of edges) for validation, following the work \cite{zhang2021magnet}. 
Performance is assessed by measuring the Area Under the ROC Curve (AUC).
The link prediction quality is computed by the average performance and standard deviation over 10 random splits.
We utilize the source code released by the authors for the baseline algorithms and optimize their hyperparameters using Optuna \cite{akiba2019optuna}.

\end{document}